\documentclass[Afour,sageh,times]{sagej}

\usepackage{moreverb,url}

\newcommand\BibTeX{{\rmfamily B\kern-.05em \textsc{i\kern-.025em b}\kern-.08em
T\kern-.1667em\lower.7ex\hbox{E}\kern-.125emX}}

\def\volumeyear{2016}

\usepackage{amsmath,amssymb,amsopn,amstext,amsfonts, amsthm}
\usepackage{amsmath,amsopn,amstext,amsfonts,mathtools}
\usepackage{booktabs}
\usepackage{xcolor}
\usepackage{threeparttable}
\usepackage{yfonts}
\newtheorem{theorem}{Theorem}
\newtheorem{lemma}{Lemma}

\newtheorem{corollary}{Corollary}
\newtheorem{remark}{Remark}

\usepackage{multirow}
\usepackage{bbm}
\usepackage{color}\definecolor{white}{rgb}{0.99,0.99,0.99}\definecolor{gray}{rgb}{0.96,0.96,0.96}
\usepackage{soul}
\usepackage[ruled,vlined,linesnumbered]{algorithm2e}
\usepackage{multirow}
\usepackage{pdfpages}
\usepackage{graphicx}
\usepackage{booktabs}

\usepackage[colorlinks,bookmarksopen,bookmarksnumbered,citecolor=blue,urlcolor=red]{hyperref}

\allowdisplaybreaks

\begin{document}

\runninghead{Li \textit{et al}.}

\title{Efficient and Distributed Large-Scale Point Cloud Bundle Adjustment via Majorization-Minimization}

\author{Rundong Li$^1$, Zheng Liu$^1$, Hairuo Wei$^1$, Yixi Cai$^1$, Haotian Li$^1$ and Fu Zhang$^1$}

\affiliation{\affilnum{1}The Department of Mechanical Engineering, The University of Hong Kong, Hong Kong. {\tt\footnotesize $\{$rdli10010,u3007335,hairuo, yixicai, haotianl$\}$@connect.hku.hk}, {\tt\footnotesize fuzhang@hku.hk}}
\corrauth{Fu Zhang, Mechatronics and Robotic Systems Laboratory, Department of Mechanical Engineering, The University of Hong Kong, HW 7-18,  Pokfulam, Hong Kong.}
\email{fuzhang@hku.hk}

\begin{abstract}
Point cloud bundle adjustment is critical in large-scale point cloud mapping.
However, point cloud bundle adjustment is both computationally and memory intensive, with its complexity growing cubically as the number of scan poses increases.
This paper presents {BALM3.0,} an efficient and distributed large-scale point cloud bundle adjustment method. The proposed method employs the majorization-minimization algorithm to decouple the scan poses in the bundle adjustment process, thus performing the point cloud bundle adjustment on large-scale data with improved computational efficiency.
The key difficulty of applying majorization-minimization on bundle adjustment is to identify the proper upper surrogate cost function.
In this paper, the proposed upper surrogate cost function is based on the point-to-plane distance.
The primary advantages of decoupling the scan poses via a majorization-minimization algorithm stem from two key aspects.
First, the decoupling of scan poses reduces the optimization time complexity from cubic to linear, significantly enhancing the computational efficiency of the bundle adjustment process in large-scale environments.
Second, it lays the theoretical foundation for distributed bundle adjustment. By distributing both data and computation across multiple devices, this approach helps overcome the limitations posed by large memory and computational requirements, which may be difficult for a single device to handle.\\
The proposed method is extensively evaluated in both simulated and real-world environments.
The results demonstrate that the proposed method achieves the same optimal residual with comparable accuracy while offering up to 704 times faster optimization speed and reducing memory usage to 1/8.
Furthermore, this paper also presented and implemented a distributed bundle adjustment framework and successfully optimized large-scale data (21,436 poses with 70 GB point clouds) with four consumer-level laptops.
\end{abstract}


\maketitle

\section{Introduction}
Large-scale point cloud maps are important in fields such as robotics \cite{hba, krusi2017driving}, city modeling \cite{lafarge2012creating, huang2022city3d, wang2018lidar}, and Geographic Information Systems (GIS) \cite{schwalbe20053d, boyko2011extracting}. With a wide detection range and accurate direct 3D measurements, LiDAR can be used to construct dense and precise point cloud maps efficiently.
Constructing accurate and globally consistent point cloud maps with LiDAR relies heavily on the precise registration of multiple LiDAR scans (also known as multi-view registration \cite{bergevin1996towards, lu1997globally, pulli1999multiview, huber2003fully, borrmann2008globally, govindu2013averaging}).
Typically, a point cloud bundle (plane) adjustment method is employed to jointly estimate the LiDAR poses and map parameters, achieving consistency and accuracy in the mapping process.

Recently, robotics researchers have proposed various
approaches to address the bundle adjustment problem for LiDAR point clouds \cite{ferrer2019eigen, liu2021balm, zhou2021pi, huang2021bundle, liu2023efficient}. These methods construct a cost function (e.g., point-to-plane, point-to-edge, plane-to-plane) as a metric and directly optimize it concerning LiDAR poses and map parameters, typically using a second-order
optimizer like the Levenberg-Marquardt algorithm  \cite{liu2021balm, liu2023efficient}. These approaches ensure global
consistency and high accuracy and have been successful in small-scale mapping tasks and sliding-window-based SLAM systems \cite{liu2024voxel}.

However, when the optimization scale increases, these methods face challenges due to time and memory costs.
Due to overlapping observations between different LiDAR scans, their scan poses in bundle adjustment are coupled, resulting in a dense Hessian matrix whose dimension is proportional to the number of scan poses. This Hessian matrix is then involved in a linear equation and {solved} repeatedly in the Levenberg-Marquardt algorithm to determine the optimal scan poses \cite{liu2021balm, liu2023efficient}.
Since the computation time for {solving} the Hessian matrix is cubic to its dimension {(i.e., LiDAR poses)} \cite{allaire2008numerical}, the bundle adjustment problem quickly becomes prohibitive for large-scale scenes with hundreds of thousands LiDAR scans, where the required computation and memory resources are beyond the affordability of individual computers.
Although some approaches have attempted to combine bundle adjustment or pairwise registration methods with pose graph optimization (PGO) \cite{hba, gpugicp, pang2024lm}, the PGO only considers constraints from registered relative scan poses, while the mapping consistency indicated by the raw points is completely ignored. As a result, these approaches need to introduce a large number of redundant registration pairs \cite{gpugicp} or iteratively construct and optimize the pose graph \cite{hba}, which increases both time and memory consumption, ultimately leading to a declining {efficiency overall}.

In addition to the cubic time growth of the bundle adjustment process, the high storage overhead caused by the dense point measurements and their processing cannot be ignored \cite{yu2024slim, pang2024lm}.
When the area of the mapping environment expands, the number of LiDAR scans that need to be stored, maintained, and optimized will exceed the capacity of a single computing device.
In such cases, a practical approach is to distribute the data across multiple computing devices and optimize it using corresponding distributed bundle adjustment algorithms. However, the coupling of scan poses in bundle adjustment presents a fundamental challenge for distributing both data storage and optimization. While multiple attempts have been made on visual bundle adjustment \cite{ren2022megba, demmel2020distributed, zhou2020stochastic, mmvba}, little research has focused on point cloud bundle adjustment.

Recently, \cite{mmvba} successfully decoupled the state variables with a majorization-minimization method in visual bundle adjustment, improving the time efficiency {by} hundreds of times.
In this paper, we have revealed that the LiDAR scan poses in the point cloud bundle adjustment process can also be completely decoupled via the majorization-minimization method.
The majorization-minimization method simplifies complex optimization problems by constructing a surrogate function that is easier to minimize. The surrogate function is designed to be an upper bound of the original cost function, so any improvements on the surrogate function are ensured to improve the original cost as well. By iteratively constructing and minimizing this surrogate function, the method gradually converges to the minimum of the original problem.

In this paper, we propose such a surrogate cost function of the point-to-plane cost {in the point cloud bundle adjustment problem formulation} \cite{liu2023efficient}.
Unlike the original cost function, the proposed surrogate function decouples all the LiDAR scan poses, offering two significant advantages.
First, the decoupled LiDAR poses reduce the optimization time from cubic to linear, significantly improving the computational efficiency of bundle adjustment in large-scale environments.
Second, decoupling the LiDAR poses enables performing the bundle adjustment process across multiple devices. By distributing the data and computation, we can overcome the limitations of memory and computational resources that a single device may struggle to handle.
The main contributions of this paper are summarized as follows:
\begin{itemize}
    \item We {propose} an upper surrogate cost function for the point-to-plane residual in point cloud bundle adjustment, which allows us to decouple LiDAR poses during the bundle adjustment process.
    \item We formally {prove} the convergence of the proposed upper surrogate cost function and {provide} analytical expressions for its first-order and second-order derivatives to facilitate Levenberg-Marquardt (LM) optimization.
    \item We {design} both single-device and distributed bundle adjustment frameworks for large-scale mapping, which, to the best of our knowledge, is the first distributed point cloud bundle adjustment framework.
    \item We {implement} the proposed surrogate cost function and frameworks {into a complete system using C++, named BALM3.0,} and conducted exclusive experiments to evaluate the convergence, accuracy, and efficiency of the proposed method.
\end{itemize}

The remainder of the paper is organized as follows: In Section \ref{sec.relatedWorks}, we discuss related research works.
In Section \ref{bundle_adjustment_formulation}, we present the formulation and proof of our proposed surrogate function, along with its first-order and second-order derivatives.
{Building on the proposed surrogate function, we also present a novel distributed bundle adjustment framework in this section.}
In Section \ref{sec.implementation}, we introduce the implementation details of our proposed method. In Section \ref{simulation}, we report the evaluation of the proposed method in a simulated environment, including its convergence, accuracy, and efficiency.
In Section \ref{benchmark}, we present benchmark comparisons on open datasets.
In Section \ref{sec.distributed_mapping_experiment}, we report distributed bundle adjustment experiments conducted on multiple devices. Finally, in Section \ref{sec.conclusion}, we provide the conclusions.

\section{Related Works} \label{sec.relatedWorks}

\subsection{Point Cloud Bundle Adjustment}

In recent years, researchers in the robotics community have begun addressing multi-view scan registration problems by developing point cloud Bundle Adjustment (BA) techniques.
Kaess \cite{pa, slamip, kfpslam} utilizes plane features from point clouds and proposes a plane adjustment method. This method minimizes the point-to-plane distance across all scan poses and plane parameters using Levenberg–Marquardt (LM) optimization.
Ferrer \cite{ferrer2019eigen} exploits a similar point-to-plane distance approach as used in \cite{pa}, but eliminates the plane parameters from the optimization process. This reduction simplifies the optimization to finding the minimum eigenvalue of a covariance matrix, considering only the scan poses. As a result, the method is termed ``eigen-factor.'' Ferrer \cite{ferrer2019eigen} then utilizes a first-order solver to solve the resulting optimization problem.

Later, our work \cite{liu2021balm} introduced BALM, which takes a further step toward more efficient BA. Similar to \cite{pa, ferrer2019eigen}, BALM minimizes the natural point-to-plane distance and analytically solves the plane parameters using closed-form solutions before the numerical optimization. As a result, the optimization problem depends only on the scan poses, with the large number of plane parameters completely eliminated. BALM then derives the analytical Jacobian and Hessian, and develops an efficient second-order solver to address the optimization problem. Compared to \cite{pa, slamip, kfpslam}, BALM achieves significantly lower optimization dimensionality (and thus increased efficiency) by removing plane parameters from the optimization variables. Additionally, compared to Ferrer \cite{ferrer2019eigen}, BALM demonstrates much faster convergence due to the use of a second-order solver.

To further enhance the computational efficiency of \cite{liu2021balm}, our latest work, BALM2 \cite{liu2023efficient}, introduces the concept of point clusters, which eliminates the need to enumerate each individual raw point when calculating the cost function, Jacobian, and Hessian matrices as done in \cite{liu2021balm}. BALM2 then derives the analytical forms of the cost function, Jacobian, and Hessian matrix using the point cluster representation and develops an efficient second-order solver for bundle adjustment optimization. As demonstrated in \cite{liu2023efficient}, BALM2 is currently the most efficient point cloud bundle adjustment technique available.

Despite these advancements, the efficiency of \cite{liu2023efficient} remains inadequate for large-scale environments. The primary bottleneck lies in the cubic time complexity of its second-order optimizer, leading to prohibitive time costs as the number of scan poses increases. In this paper, the proposed method effectively addresses this issue by employing the majorization-minimization approach to decouple all scan poses in the bundle adjustment problem.

\subsection{Majorization-Minimization Approach}
The majorization-minimization approach is widely used in fields such as signal processing \cite{figueiredo2007majorization, bioucas2006total}, communications \cite{sun2016majorization, gong2020majorization}, and machine learning \cite{mairal2015incremental}. This method simplifies complex optimization problems by constructing a surrogate function that is easier to minimize. The surrogate function serves as an upper bound to the original cost function, ensuring that any improvement in the surrogate function also improves the original cost function. By iteratively constructing and minimizing the surrogate function, the method gradually converges to a minimum of the original problem.

In the field of {Simultaneous Localization and Mapping (SLAM)} research, the majorization-minimization algorithm was first applied to Pose Graph Optimization (PGO). \cite{fan2020majorization} proposed a surrogate cost function for the maximum likelihood estimation of the pose graph. In their proposed surrogate function, the state variables are decoupled, enabling distributed optimization. Their more recent work \cite{fan2023majorization} extends these preliminary results and achieves decentralized optimization. On the other hand, \cite{mmvba} successfully applied the majorization-minimization algorithm to visual bundle adjustment. In \cite{mmvba}, the reprojection error for visual bundle adjustment was first reformulated, allowing the depth estimation for each observation to be expressed analytically. Based on this reformulated reprojection error, they proposed a surrogate function in which the camera poses and map points are decoupled.

Inspired by \cite{mmvba}, in this paper, we further explore the potential of applying the majorization-minimization approach to point cloud bundle adjustment and propose a surrogate cost function in which the LiDAR scan poses are completely decoupled. As a result, our method significantly improves the optimization speed of point cloud bundle adjustment on large-scale data. Furthermore, it supports optimization across multiple devices, enabling distributed bundle adjustment.

\subsection{Large-scale Point Cloud Mapping}
Due to the gap between increasing requirements for large-scale environment mapping and the efficiency limitations of multi-view registration, including bundle adjustment for point cloud data, large-scale point cloud mapping has attracted significant attention in the research community.
Given the cubic time complexity of state-of-the-art point cloud bundle adjustment problems, most existing approaches aim to solve the large-scale point cloud mapping problem through a combination of pairwise or multi-view registration and pose graph optimization \cite{gpugicp, hba, pang2024lm}.

In \cite{hba}, PGO was combined with the point cloud bundle adjustment method from \cite{liu2023efficient}, resulting in a bottom-up hierarchical approach. In this framework, the method described in \cite{liu2023efficient} is applied exclusively within a sliding window, thereby mitigating significant time costs. To ensure map consistency, the framework is iteratively executed multiple times.
The work in \cite{gpugicp} combined the Generalized Iterative Closest Point (GICP) algorithm with PGO. To ensure map consistency during PGO optimization, they introduced a large number of GICP pairs into the pose graph and utilized a GPU to compute pairwise registrations in parallel, addressing the efficiency challenges posed by the large number of GICP registrations.
Additionally, \cite{pang2024lm} proposed a pose-graph-based hierarchical optimization framework. They partitioned the original point cloud data into smaller map blocks to reduce memory load and further enhanced the adaptive voxelization process from \cite{liu2021balm} by removing low-quality plane features, thereby improving both efficiency and accuracy.

{Rather than working around the high computation complexity of large scale point cloud bundle adjustment as in the aforementioned methods by combining with PGO, this paper formally attacks the computation complexity problem of computation efficiency of large scale point cloud bundle adjustment. The substantial efficiency improvements achieved by the proposed method eliminate the dependence on pose graph optimization, thereby avoiding the redundant computations typically associated with PGO. Moreover, the method is inherently compatible with existing mapping systems, such as \cite{hba, pang2024lm}, offering the potential to further enhance their performance.}

\section{Upper Surrogate Cost Formulation and Optimization} \label{bundle_adjustment_formulation}
In this chapter, we derive our upper surrogate cost function formulation and optimization. First, we briefly introduce the preliminaries including the point-to-plane distance and principle of the majorization-minimization method (Section. \ref{sec.preliminaries}). Then we introduce the proposed upper surrogate cost function, the corresponding first-order and second-order derivatives, and its optimization (Section. \ref{sec.BAformulation}). Throughout this paper, we use notations summarized in Table \ref{table.notation} or otherwise specified in the context.
\begin{table}[t]
	\caption{Nomenclatures} \label{table.notation}
	\centering
    \resizebox{\columnwidth}{!}{
    \begin{tabular}{rp{7cm}}
		\toprule
		Notation & \qquad\qquad\qquad\qquad Explanation \\
		\midrule
		$\mathbb R^{m \times n}$ & The set of $m \times n$ real matrices. \\
		$\mathbb S^{m \times m}$ & The set of $m \times m$ symmetric matrices. \\
		$\boxplus$  & The encapsulated ``boxplus'' operations on manifold.\\
		$\lambda_{min}(\cdot)$ & The minimal eigenvalue of the matrix $(\cdot)$. \\
            $\mathbf u_{min}(\cdot)$ & The eigenvector corresponding to the minimal eigenvalue of the matrix $(\cdot)$. \\
            $(\cdot)^{(k)}$ & The values of $(\cdot)$ calculated with the states at the $k$-th iteration. \\
            $\lfloor \cdot \rfloor$  & The skew symmetric matrix of $(\cdot)$. \\
            % $\nabla c(\cdot)$ & The gradient of function $c(\cdot)$ with respect to $(\cdot)$. \\
		\bottomrule 
	\end{tabular}}
\end{table}

\subsection{Preliminaries} \label{sec.preliminaries}
\subsubsection{Point-to-Plane Cost Function}
Assume there are $M_f$ plane features, each denoted by $\boldsymbol{\pi}_i$ ($i=1,...,M_f$), and $M_p$ LiDAR scans, each pose denoted by $\mathbf T_j = (\mathbf R_j, \mathbf t_j)$ ($j=1,...,M_p$). The bundle adjustment problem is formulated as minimizing the natural distance between each LiDAR raw point to its corresponding plane \cite{pa}. Analytically solving the plane parameters $\boldsymbol{\pi}_i$ and using the point cluster representation, our previous work \cite{liu2023efficient} has formulated the bundle adjustment problem as (see (22) in \cite{liu2023efficient}) 

\begin{align}
    \min_{\mathbf T_j \in SE(3), \forall j} c(\mathbf T) &= \sum_{i=1}^{M_f} \underbrace{ \lambda_{\min} \left( \frac{1}{N_i} \mathbf P_i - \frac{1}{N_i^2} \mathbf v_i \mathbf v_i^T \right)}_{c_i (\mathbf T)}, \; \nonumber \\
    \quad \mathbf P_i &= \sum_{j=1}^{M_p} \mathbf P_{ij}, \mathbf v_i = \sum_{j=1}^{M_p} \mathbf v_{ij}  \label{BA-formulation-reduced-reduced}
\end{align}
where $\lambda_{\text{min}} ( \cdot )$ represents the minimal eigenvalue of the input matrix, $N_i$ is the total number of raw points on the $i$-th plane feature observed by all $M_p$ LiDAR scans, and 
\begin{align}
    {\mathbf P}_{ij} &= \mathbf R_j \mathbf P_{f_{ij}} \mathbf R_j^T + \mathbf R_j \mathbf v_{f_{ij}} \mathbf t_j^T + \mathbf t_j \mathbf v_{f_{ij}}^T \mathbf R_j^T + N_{ij} \mathbf t_j \mathbf t_j^T, \; \nonumber \\
    \quad {\mathbf v}_{ij} &= \mathbf R_j \mathbf v_{f_{ij}} + N_{ij} \mathbf t_j \; \nonumber \\
    \mathbf P_{f_{ij}} &= \sum_{k=1}^{N_{ij}} \mathbf p_{f_{ijk}} \mathbf p_{f_{ijk}}^T, \quad
    \mathbf v_{f_{ij}} = \sum_{k=1}^{N_{ij}} \mathbf p_{f_{ijk}} \label{feature-pc}
\end{align}
with $N_{ij}$ being the total number of raw points on the $i$-the plane feature observed by the $j$-th scan, and $\mathbf p_{f_{ijk}}$ being the $k$-th raw point of them.  Notice that the point clusters,  $\mathbf P_{f_{ij}}$ and $\mathbf v_{f_{ij}}$, remain independent from the scan poses $\mathbf T_j$, hence need to be calculated only once prior to the bundle adjustment optimization. After that, the cost function in (\ref{BA-formulation-reduced-reduced}) can be calculated from the pre-computed point clusters without enumerating each raw points $\mathbf p_{f_{ijk}}$.

\subsubsection{Majorization-Minimization Method} \label{sec.majorization_minimization}
The majorization-minimization approach seeks to find a surrogate cost function at each step of the optimization iteration. Denote $c_M(\mathbf{T}|\mathbf{T}^{(k)})$ the surrogate function  at the $k$-th iteration with the optimal solution $\mathbf{T}^{(k)}$, a valid surrogate function should satisfy:
\begin{align} \label{eq.mm_req1}
     c_M(\mathbf T|\mathbf T^{(k)})&\geq c(\mathbf T); \quad c_M(\mathbf T^{(k)}|\mathbf T^{(k)}) = c(\mathbf T^{(k)}) \; \nonumber \\
    &\forall \mathbf T \in SE(3) \times \cdots \times SE(3)
\end{align}
The majorization-minimization approach then optimizes the surrogate function $c_M(\mathbf{T}|\mathbf{T}^{(k)})$ (instead of the original cost function $c(\mathbf{T})$) to obtain the next iteration optimal solution $\mathbf T^{(k+1)}$. Due to the conditions in (\ref{eq.mm_req1}) and $\mathbf T^{(k+1)}$ being the optimal solution of $c_M(\mathbf{T}|\mathbf{T}^{(k)})$, we have
\begin{align} \label{eq.mm_2}
    c(\mathbf T^{(k+1)}) \leq c_M(\mathbf T^{(k+1)}|\mathbf T^{(k)})
    \leq c_M(\mathbf T^{(k)}|\mathbf T^{(k)}) = c(\mathbf T^{(k)})
\end{align}
That is, optimizing the surrogate function $c_M(\mathbf{T}|\mathbf{T}^{(k)})$ will ensure to improve the original cost function $c(\mathbf{T})$.  Therefore, $\mathbf T^{(k+1)}$ will be a valid update of the original optimization problem as well (see Figure \ref{fig.mm_theory}(a)). 

The majorization-minimization approach gives us the flexibility to choose the surrogate function, as long as the conditions (\ref{eq.mm_req1}) are met. In particularly, if the surrogate function has all the optimization variables being completely decoupled, that is,
\begin{align}
\label{eq.mm_req2}
    c_M(\mathbf T|\mathbf T^{(k)}) = \sum_{j=1}^{M_p} c_{M_j} (\mathbf T_j | \mathbf T^{(k)})
\end{align}
\normalsize
where $c_{M_j} (\mathbf T_j | \mathbf T^{(k)})$ is any function but is related to only one pose $\mathbf T_j$, then its Hessian matrix would be a (block) diagonal matrix which can be solved in linear time (see Figure \ref{fig.mm_theory} (b)). 

\begin{figure} [ht] 
	\centering
	\includegraphics[width=\linewidth]{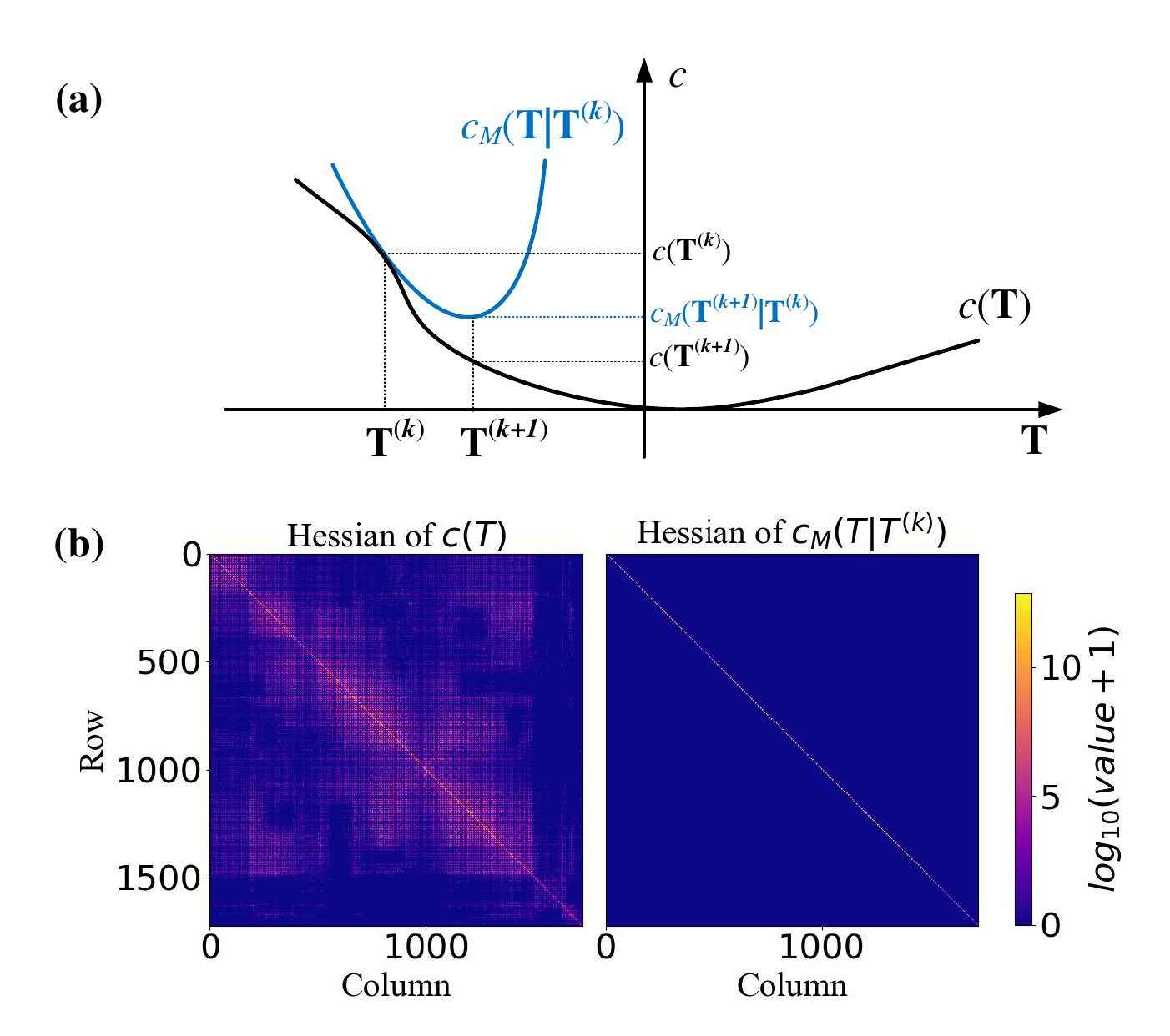}
	\caption{Explanation of key technical details in majorization-minimization: (a) Optimizing the upper surrogate cost function $c(\mathbf{T}|\mathbf{T}^{(k)})$ yields a state $\mathbf{T}^{(k+1)}$ that ensures $c(\mathbf{T}^{(k+1)}) \leq c(\mathbf{T}^{(k)})$, thus decreasing the original cost function $c(\mathbf{T})$. (b) Compared to the original cost function, the upper surrogate cost function features a block-diagonal Hessian matrix, reducing the solving time from cubic to linear.}
    \label{fig.mm_theory}
\end{figure}
\subsection{Upper Surrogate Cost Function} \label{sec.BAformulation}
According to the preliminaries mentioned in Sec. \ref{sec.preliminaries}. Our goal is to find a majorized cost item $c_{M}(\mathbf T | \mathbf T^{(k)})$, which should fulfill (\ref{eq.mm_req1}) and (\ref{eq.mm_req2}). It should be noted that there are multiple $c_{M}(\mathbf T | \mathbf T^{(k)})$ fulfilling the requirements. In this section, we derive one cost function that fulfills all the requirements above.
We first give Lemma \ref{lemma1}:
\begin{lemma} \label{lemma1}
For a set of scalars $x_i\in \mathbb R$ and $z\in \mathbb R$, we have:
% \vspace{-3pt}
% \footnotesize
\begin{align}
    \bigg( \sum_{i=1}^{n} x_i \bigg)^2 \ge -4 z^2 + 4 \bigg( \sum_{i=1}^{n} x_i \bigg) z
\end{align}
\normalsize
and the inequality holds if:
% \vspace{-3pt}
% \footnotesize
\begin{align}
    z^\star=\frac{1}{2}\sum_{i=1}^n x_i
\end{align}
\normalsize
\end{lemma}
\begin{proof}
See Appendix \ref{proof.lemma1}. 
\end{proof}
Based on Lemma \ref{lemma1}, we can prove the following theorem.
\begin{theorem} \label{theorem.surrogate_cost_function}
Given
\begin{itemize}
    \item[(1)] The $i$-th planar feature visualized by $M_p$ LiDAR scans with each pose denoted as $\mathbf T_j \in SE(3), j = 1, \cdots, M_p$;
    \item[(2)] The point cluster matrices $\mathbf P_{ij}, \mathbf v_{ij};\quad j = 1, \cdots, M_p;$ and the point-to-plane residual represented as (\ref{BA-formulation-reduced-reduced});
\end{itemize}
Let 
% \footnotesize
\begin{align}\label{eq:surrogate}
    c_{M_i} (\mathbf T | \mathbf T^{(k)}) \triangleq& \sum_{j=1}^{M_p} \left( \frac{1}{N_i} {\mathbf u_{\min}^{(k)}}^T {\mathbf P}_{ij} \mathbf u_{\min}^{(k)} - \frac{4 z^{(k)}}{N_i^2} {\mathbf u_{\min}^{(k)}}^T {\mathbf v}_{ij}  \right) \; \nonumber \\
    &+ \frac{4}{N_i^2} {z^{(k)}}^2
\end{align}
\normalsize
where $\mathbf u_{\min}^{(k)}$ is the eigenvector associated to the minimum eigenvalue of the matrix $\frac{1}{N_i} \mathbf P_i^{(k)} - \frac{1}{N_i^2} \mathbf v_i^{(k)} {\mathbf v_i^{(k)}}^T$ and $z^{(k)} = \frac{1}{2}\sum_{j=1}^{M_p}{\mathbf u_{\min}^{(k)}}^T {\mathbf v}_{ij}^{(k)}$ with $\mathbf P_i^{(k)} , \mathbf v_i^{(k)}, {\mathbf v}_{ij}^{(k)}$ being ${\mathbf P}_{i}, {\mathbf v}_{i}, {\mathbf v}_{ij}$, respectively, evaluated at $\mathbf T^{(k)}$. Then, $c_{M_i} (\mathbf T | \mathbf T^{(k)})$ is a valid surrogate function of the $i$-th cost item $c_i(\mathbf T)$ of (\ref{BA-formulation-reduced-reduced}) satisfying (\ref{eq.mm_req1}).
\end{theorem}

\begin{proof}
See Appendix \ref{proof.theorem1}
\end{proof}

{Based on Theorem \ref{theorem.surrogate_cost_function}, the cost function $c_M(\mathbf T | \mathbf T^{(k)})$ for the entire bundle adjustment problem is obtained by summing (\ref{eq:surrogate}) corresponding to each feature:
\begin{align} \label{eq:surrogate_total}
    c_M = \sum_{i=1}^{M_f} c_{M_i}
\end{align}}

\subsection{First and Second Order Derivatives}
Similar to \cite{liu2023efficient}, we employed a Levenberg-Marquardt optimizer to minimize the cost function which relies on the first-order and second-order derivatives to determine the optimal iteration step. Thus, the first-order and the second-order derivatives (the Jacobian and Hessian matrix) are required. In this section, we derive the analytic expression of such Jacobian and Hessian matrices.
To guarantee the consistency of our method and the previous approach \cite{liu2023efficient}, we utilize the SE(3) perturbation to calculate the derivatives following \cite{liu2023efficient}:
\begin{align}\label{eq:input_perturbation}
    \mathbf T \boxplus \delta \mathbf T &\triangleq (\cdots, \mathbf T_j \boxplus \delta \mathbf T_j, \cdots), \; \nonumber \\
    \mathbf T_j \boxplus \delta \mathbf T_j & \triangleq (\exp{(\lfloor \delta \boldsymbol{\phi}_j \rfloor)} \mathbf R_j, \delta \mathbf t_j + \exp{(\lfloor \delta \boldsymbol{\phi}_j \rfloor)} \mathbf t_j )
\end{align}
where $\delta \mathbf T \triangleq (\cdots, \delta \mathbf T_j, \cdots) \in \mathbb{R}^{6M_p}$ with $\delta \mathbf T_j \triangleq (\delta \boldsymbol \phi_j, \delta \mathbf t_j) \in \mathbb{R}^6, \forall j, $ is the perturbation on the pose vector. 
{Let $\left( \frac{\partial c_M(\mathbf{T} \mid \mathbf{T}^{(k)})}{\partial \mathbf{T}} \right)(\mathbf{T}_0)$ and $\left( \frac{\partial^2 c_M(\mathbf{T} \mid \mathbf{T}^{(k)})}{\partial \mathbf{T}^2} \right)(\mathbf{T}_0)$ denote the first-order and second-order derivatives of $c_M(\cdot)$, respectively, evaluated at the chosen point $\mathbf{T}_0$. The $\boxplus$ operation allows us to parameterize the input $\mathbf{T}$ of the function $c_M(\cdot)$ in terms of a perturbation $\delta \mathbf{T}$ around a given base point $\mathbf{T}_0$, such that $\mathbf{T} = \mathbf{T}_0 \boxplus \delta \mathbf{T}$. Consequently, the derivatives can be expressed as:
\footnotesize
\begin{align} \label{eq.derivatives_reform}
    &\left( \frac{\partial c_M(\mathbf T | \mathbf T^{(k)})}{\partial \mathbf T} \right) (\mathbf T_0) \triangleq \left(\frac{\partial c_M(\mathbf T_0 \boxplus \delta \mathbf T| \mathbf T^{(k)})}{\partial \delta \mathbf T } \right) ( \mathbf 0) \; \nonumber \\
    &\left( \frac{\partial^2 c_M(\mathbf T | \mathbf T^{(k)})}{\partial \mathbf T^2} \right) (\mathbf T_0) \triangleq \left( \frac{\partial }{\partial \delta \mathbf T } {\left( \frac{\partial c_M(\mathbf T_0 \boxplus \delta \mathbf T | \mathbf T^{(k)})}{\partial \delta \mathbf T} \right)}  \right) (\mathbf 0) \; \nonumber \\
    &\quad \quad \quad \quad \quad \quad \quad \quad \quad \quad \forall \mathbf T_0 \in SE(3) \times \cdots \times SE(3).
\end{align}
\normalsize
}
Based on (\ref{eq:input_perturbation}) and (\ref{eq.derivatives_reform}), The Jacobian and Hessian matrices are derived in Theorem \ref{theorem.jandh}.

\begin{theorem} \label{theorem.jandh}
    Given
\begin{itemize}
    \item[(1)] Poses $\mathbf T_j \in SE(3), j = 1, \cdots, M_p$ with the value at the $k$-th iteration denoted as $\mathbf T_j^{(k)}$; and
    \item[(2)] Upper surrogate cost function $c_{M_i}(\mathbf T|\mathbf T^{(k)})$ mentioned in (\ref{eq:surrogate});
\end{itemize}
the Jacobian matrix $\mathbf J_i$ and Hessian matrix $\mathbf H_i$ of the function $c_{M_i}(\mathbf T|\mathbf T^{(k)})$ with respect to the poses $\mathbf T$ are
\begin{align}
\label{eq.jacob_and_hess}
    \mathbf J_i &= \begin{bmatrix}
        \mathbf J_{i0} & \mathbf J_{i1} & \cdots &  \mathbf J_{iM_p}
    \end{bmatrix} \; \nonumber \\
   \mathbf H_i &= {\mathrm{diag}}(\mathbf H_{i1}, \mathbf H_{i2}, \cdots , \mathbf H_{iM_p})
\end{align}
where
\begin{align}
    \mathbf J_{ij} =& \begin{bmatrix}
        \mathbf g_\phi & \mathbf g_t
    \end{bmatrix} \; \nonumber \\
    \mathbf H_{ij} =& \begin{bmatrix}
        \mathbf h_{\phi \phi} & \mathbf h_{\phi t} \\
        \mathbf h_{\phi t}^T & \mathbf h_{tt}
    \end{bmatrix}
\end{align}
and
\begin{align}
    {\mathbf g_{\phi}} =& \frac{2}{N_i} {\mathbf u_{\text{min}}^{(k)}}^T {\mathbf R_j} \mathbf P_{f_{ij}} {\mathbf R_j}^T \lfloor \mathbf u_{\text{min}}^{(k)} \rfloor + 
    \frac{4z^{(k)}}{N_i^2} {\mathbf u_{\text{min}}^{(k)}}^T  \lfloor {\mathbf R_j} \mathbf v_{f_{ij}} \rfloor \; \nonumber \\
    &+ \frac{2}{N_i} {\mathbf u_{\text{min}}^{(k)}}^T \mathbf t_j {\mathbf v_{f_{ij}}}^T {\mathbf R_j}^T \lfloor \mathbf u_{\text{min}}^{(k)} \rfloor \; \nonumber \\
    &+ \frac{2}{N_i} {\mathbf u_{\text{min}}^{(k)}}^T \mathbf R_j \mathbf v_{f_{ij}} {\mathbf t_j}^T \lfloor \mathbf u_{\text{min}}^{(k)} \rfloor \; \nonumber \\
    &+ \frac{2 N_{ij}}{N_i}  {\mathbf u_{\text{min}}^{(k)}}^T \mathbf t_j {\mathbf t_j}^T \lfloor {\mathbf u_{\text{min}}^{(k)}} \rfloor - \frac{4z^{(k)} N_{ij}}{N_i^2} { \mathbf t_j }^T \lfloor \mathbf u_{\text{min}}^{(k)} \rfloor \; \nonumber \\
    {\mathbf g_t} =& \frac{2}{N_i} {\mathbf u_{\text{min}}^{(k)}}^T \mathbf R_j \mathbf v_{f_{ij}} {\mathbf u_{\text{min}}^{(k)}}^T \; \nonumber \\
    &+ \frac{2 N_{ij}}{N_i} {\mathbf t_j}^T {\mathbf u_{\text{min}}^{(k)}} {\mathbf u_{\text{min}}^{(k)}}^T  - \frac{4z^{(k)}N_{ij}}{N_i^2} {\mathbf u_{\text{min}}^{(k)}}^T \; \nonumber \\
    {\mathbf h_{\phi \phi}} =& -\frac{2}{N_i} \lfloor \mathbf u_{\text{min}}^{(k)} \rfloor {\mathbf R_j} \mathbf P_{f_{ij}} {\mathbf R_j}^T \lfloor \mathbf u_{\text{min}}^{(k)} \rfloor \; \nonumber \\
    & + \frac{1}{N_i} \lfloor \mathbf u_{\text{min}}^{(k)} \rfloor \lfloor \mathbf R_j \mathbf P_{f_{ij}} {\mathbf R_j}^T \mathbf u_{\text{min}}^{(k)} \rfloor \; \nonumber \\
    &+ \frac{1}{N_i} \lfloor \mathbf R_j \mathbf P_{f_{ij}} {\mathbf R_j}^T \mathbf u_{\text{min}}^{(k)} \rfloor \lfloor \mathbf u_{\text{min}}^{(k)} \rfloor \; \nonumber \\ 
    &- \frac{2 z^{(k)}}{N_i^2} \lfloor \mathbf u_{\text{min}}^{(k)} \rfloor \lfloor \mathbf R_j \mathbf v_{f_{ij}} \rfloor - \frac{2 z^{(k)}}{N_i^2} \lfloor \mathbf R_j \mathbf v_{f_{ij}} \rfloor \lfloor \mathbf u_{\text{min}}^{(k)} \rfloor \; \nonumber \\
    &+ \frac{2}{N_i} \lfloor {\mathbf u_{\text{min}}^{(k)}} \rfloor \lfloor \mathbf R_j \mathbf v_{f_{ij}} {\mathbf t_j}^T \mathbf u_{\text{min}}^{(k)} \rfloor \; \nonumber \\
    &+ \frac{2}{N_i} \lfloor \mathbf t_j {\mathbf v_{f_{ij}}}^T {\mathbf R_j}^T \mathbf u_{\text{min}}^{(k)} \rfloor \lfloor {\mathbf u_{\text{min}}^{(k)}} \rfloor \; \nonumber \\
    &- \frac{4}{N_i} \lfloor {\mathbf u_{\text{min}}^{(k)}} \rfloor \mathbf R_j \mathbf v_{f_{ij}} {\mathbf t_j}^T \lfloor \mathbf u_{\text{min}}^{(k)} \rfloor \; \nonumber \\
    & -\frac{2 N_{ij}}{N_i} \lfloor \mathbf t_j \rfloor \mathbf u_{\text{min}}^{(k)} {\mathbf u_{\text{min}}^{(k)}}^T \lfloor \mathbf t_j \rfloor \; \nonumber \\
    &+ \frac{N_{ij}}{N_i} \lfloor \mathbf t_j \rfloor \lfloor \mathbf u_{\text{min}}^{(k)} {\mathbf u_{\text{min}}^{(k)}}^T \mathbf t_j \rfloor \; \nonumber \\
    &+ \frac{N_{ij}}{N_i} \lfloor \mathbf u_{\text{min}}^{(k)} {\mathbf u_{\text{min}}^{(k)}}^T \mathbf t_j \rfloor \lfloor \mathbf t_j \rfloor- \frac{4 z^{(k)} N_{ij}}{N_i^2} \lfloor \mathbf u_{\text{min}}^{(k)} \rfloor \lfloor \mathbf t_j \rfloor \; \nonumber \\
    {\mathbf h_{tt}} =& \frac{2 N_{ij}}{N_i} {\mathbf u_{\text{min}}^{(k)}} {\mathbf u_{\text{min}}^{(k)}}^T \; \nonumber \\
    {\mathbf h_{\phi t}} =& - \frac{2}{N_i} \lfloor {\mathbf u_{\text{min}}^{(k)}} \rfloor \mathbf R_j \mathbf v_{f_{ij}} {\mathbf u_{\text{min}}^{(k)}}^T + \frac{2 N_{ij}}{N_i} \lfloor \mathbf t_j \rfloor \mathbf u_{\text{min}}^{(k)} {\mathbf u_{\text{min}}^{(k)}}^T\; \nonumber \\
\end{align}
\end{theorem}
\begin{proof}
    See Appendix \ref{proof.theorem2}.
\end{proof}
{Based on Theorem \ref{theorem.jandh}, the Jacobian matrix $\mathbf{J}$ and the Hessian matrix $\mathbf{H}$ for the entire bundle adjustment problem are obtained by summing the Jacobian and Hessian matrices corresponding to each feature:
\begin{align}
    \mathbf J = \sum_{i=1}^{M_f} \mathbf J_i; \qquad \mathbf H = \sum_{i=1}^{M_f} \mathbf H_i
\end{align}
Since each $\mathbf H_i$ is block diagonal, the $\mathbf H$ is also block diagonal matrix.}

\begin{corollary} \label{sec.corollary1}
{Given
\begin{itemize}
    \item[(1)] Original cost function $c(\mathbf{T})$ mentioned in (\ref{BA-formulation-reduced-reduced}); and
    \item[(2)] Upper surrogate cost function $c_{M}(\mathbf T|\mathbf T^{(k)})$ mentioned in (\ref{eq:surrogate_total});
\end{itemize}
Their first-order derivatives follow that:
\begin{align} \label{eq.jacob_eq}
    \left( \frac{\partial c_M(\mathbf T|\mathbf T^{(k)})}{\partial \mathbf T} \right) \left( \mathbf T^{(k)} \right) =  \left( \frac{\partial c(\mathbf T)}{\partial \mathbf T} \right) \left( \mathbf T^{(k)} \right)
\end{align}}
\end{corollary}
\begin{proof}
    {See Appendix \ref{proof.corollary1}}
\end{proof}
\begin{remark}
    {The original cost function (\ref{BA-formulation-reduced-reduced}) and the surrogate cost function (\ref{eq:surrogate_total}) share the same Jacobian matrix at $\mathbf T^{(k)}$.}
\end{remark}

\subsection{Second-order Solver}
With analytic Jacobin and Hessian matrices given in Theorem \ref{theorem.jandh}, we then minimize the surrogate cost function (\ref{eq:surrogate}) with the Levenberg-Marquardt algorithm, and the original cost function will be improved indirectly as mentioned in (\ref{eq.mm_2}). Within each iteration, the optimal iteration step $\Delta \mathbf T$ is determined by:
\begin{align}
    \left( \mathbf H + \mu \mathbf I \right) \Delta \mathbf T = -\mathbf J^T,
\end{align}
Note that the Hessian matrix $\mathbf H$ is block diagonal as mentioned in (\ref{eq.jacob_and_hess}). The equation can be further written as:
\begin{align}
    {\mathrm{diag}}(\cdots, \mathbf H_j + \mu \mathbf I, \cdots) \Delta \mathbf T = -\begin{bmatrix}
       \cdots & \mathbf J_j& \cdots
    \end{bmatrix}^T
\end{align}
where $\mathbf J_j \in \mathbb R^6$ and $\mathbf H_j \in \mathbb S^{6\times 6}$ is the Jacobian and Hessian matrices correspond to the $j$-th pose, respectively. Then the equation can be separated into $M_p$ sub-problems:
\begin{align} \label{eq.seperated_linear_equation}
    \left( \mathbf H_j + \mu \mathbf I \right)& \Delta \mathbf T_j = - \mathbf J_j^T; \quad j=1,2,\cdots,M_p
\end{align}
where $\Delta \mathbf T_j$ is the optimal iteration step of the $j$-th pose. Remarkably, each of the sub-problems is only ranked 6 and independent from the other sub-problems, in which case they can be solved parallel in a short time.
And the whole optimization algorithm is illustrated in Algorithm \ref{algo-LM}.
\begin{algorithm} [t]
	\SetAlgoLined
	\caption{LM optimization}\label{algo-LM}
	\SetKwInOut{Input}{Input}
	\Input{Initial poses $\mathbf T^{(0)}$; \\
		   Point clusters in the local frame $\mathbf P_{f_{ij}}, \mathbf v_{f_{ij}}$;}
	$k=0$;\\
	\Repeat{$\parallel \mathbf T^{(k)} \boxminus \mathbf T^{(k-1)} \parallel < \epsilon$ \rm{\textbf{or}} $k \ge k_{\text{max}}$}
	{
            $\mu = 0.01$, $\nu = 2$, $j=0$;\\
		\Repeat{$\parallel \Delta \mathbf T\parallel < \epsilon$ \rm{\textbf{or}} $c_M - c_M' > c - c'$ \rm{\textbf{or}} $j \ge j_{\text{max}}$}
            {
                $j=j+1$; \\
    		$\mathbf J = \mathbf 0_{1 \times 6M_p}, \mathbf H = \mathbf 0_{6M_p \times 6M_p}$;  \\
    		\ForEach{$i \in \{1, \cdots, M_f\}$}
    		{
    			Compute $\mathbf J_i$ and $\mathbf H_i$ from (\ref{eq.jacob_and_hess}); \\
    			$\mathbf J = \mathbf J + \mathbf J_i$; $\mathbf H = \mathbf H + \mathbf H_i$
    		}
    		Solve $(\mathbf H + \mu\mathbf I)\Delta \mathbf T = -\mathbf J^T$ by (\ref{eq.seperated_linear_equation}); \\
    		$\mathbf T' = \mathbf T \boxplus \Delta \mathbf T$; \\
    		Compute current cost $c_M = c_M(\mathbf T | \mathbf T^{(k)})$ and the new cost $c_M' = c_M(\mathbf T'|\mathbf T^{(k)})$ from (\ref{eq:surrogate});\\
                Compute current cost $c = c(\mathbf T)$ and the new cost $c' = c(\mathbf T')$ from (\ref{BA-formulation-reduced-reduced});\\
    		$\rho = (c_M - c_M') / (\frac{1}{2}\Delta \mathbf T \cdot (\mu\Delta \mathbf T - \mathbf J^T ))$; \\
    		\eIf{$\rho > 0$}
    		{
    			$\mathbf T = \mathbf T'$; \\
    			$\mu = \mu*\max{(\frac{1}{3}, 1-(2\rho-1)^3)}$; $\nu = 2$;
    		}
    		{
    			$\mu = \mu * \nu$; $\nu = 2*\nu$;
    		}
            }
            $\mathbf T^{(k+1)}$ = $\mathbf T$; \\
            $k = k+1$
	}
	\KwOut{Final optimized states $\mathbf T^{(k)}$;}
\end{algorithm}

It is worth mentioning that it is unnecessary to iterate the LM optimizer until it converges for every constructed upper surrogate cost function. Limiting the maximum inner iteration number $j_{max}$ in our Algorithm \ref{algo-LM} could lead to less overall optimization time.
Updating the surrogate cost function in each interaction can guarantee the maximum decreasing speed of the original cost function.
\subsection{Complexity Analysis}
In this section, we analyze the time and space complexity of our proposed method compared to BALM2 \cite{liu2023efficient}.

\subsubsection{Time Complexity} We begin by evaluating the time complexity, which primarily arises from two factors: the derivative evaluation process and the linear equation solving process. For the derivative evaluation, since our method remains compatible with point clusters \cite{liu2023efficient}, it avoids the need for enumerating each raw point. The time complexity of this process is proportional to {the total number of plane features observed by all LiDAR scans}, $N_{obs}$, resulting in a time complexity of $O(N_{obs})$, which is the same as BALM2. In the worst case, depending on the sparsity of the BA problem, the time complexity could reach $O(M_p \times M_f)$, where $M_p$ and $M_f$ represent the number of poses and features, respectively. However, in most cases, $N_{obs} \ll M_p \times M_f$. For the linear equation solving process, thanks to the diagonal Hessian matrix, the original equation can be decomposed into $M_p$ sub-problems, each of which has a rank of only 6. This reduces the time complexity of the process to $O(M_p)$, which is significantly lower than the $O(M_p^3)$ complexity of BALM2. Additionally, since all sub-problems are independent, this process can be further accelerated using multi-threading.

\subsubsection{Space Complexity} The space complexity also arises primarily from two factors: the storage of point clusters and the storage of the Hessian matrix during the optimization. The map parameter storage refers to the memory required to store the parameters necessary for calculating the Jacobian and Hessian matrices. Thanks to the point clustering approach, both BALM2 and our method eliminate the need to store raw point clouds, resulting in significant memory savings. The space complexity for storing map parameters is the same as in BALM2, i.e., $O(N_{obs})$. The storage of the Hessian matrix also requires substantial memory as the optimization problem scales. Our method takes advantage of the diagonal structure of the Hessian matrix, requiring only $36M_p$ floating-point values to store the Hessian, leading to a space complexity of $O(M_p)$. This is significantly lower than the $O(M_p^2)$ complexity required by BALM2.

\subsection{Distributed Bundle Adjustment Frameworks} \label{sec.distributed_framework}
In this section, we introduce the proposed distributed bundle adjustment framework as shown in Figure \ref{fig.framework}. The proposed framework is designed to function uniformly across all computing devices, with each device storing a subset of point clouds from the entire bundle adjustment problem.

\begin{figure*} [t]
	\centering
	\includegraphics[width=\linewidth]{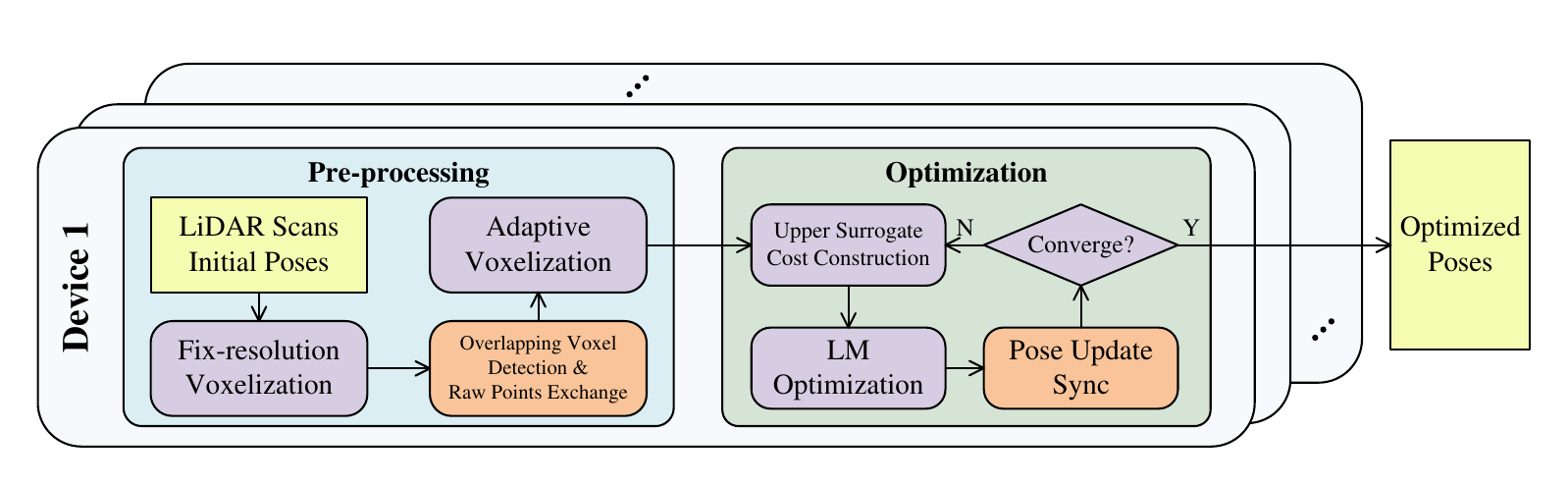}
	\caption{Pipeline of the proposed distributed bundle adjustment framework.}
    \label{fig.framework}
\end{figure*}

The framework consists of two phases: the pre-processing phase, which involves feature extraction and is executed only once, and the optimization phase, which iteratively constructs and refines the upper surrogate cost function.

In the pre-processing phase, a fixed-resolution voxelization process is performed on each device using the initial scan pose estimation, which is typically generated by an odometry algorithm. The positions of the generated voxels are shared among all devices to identify overlapping voxels. Points within these overlapping voxels are exchanged between devices.
Using the exchanged points in the overlapping voxels, each device extracts plane features, generates point clusters, and constructs both the original and surrogate cost functions.

The optimization phase operates similarly to the single-device bundle adjustment process, with the key difference being that each device optimizes only the subset of scan poses it is responsible for. The optimized scan poses are then shared among all devices, enabling the optimization to proceed iteratively across all devices.

During the whole process, our framework involves three types of data communication among devices: overlapping voxel detection, raw points exchange, and the pose update sync (highlighted with orange in Figure \ref{fig.framework}). The overlap detection only shares the voxel positions, which are compact data structures. The raw points exchange involves large data but only needs to be performed within the overlapped regions and is only performed once at the beginning of the whole process. The pose update sync is performed iteratively but only requires minimal network bandwidth. Therefore, the overall communication load and network bandwidth requirements remain within a reasonable range.

\section{Implementations}\label{sec.implementation}
We implemented our proposed method in C++ and tested it on both single-device experiments and distributed bundle adjustment experiments. For the single-device experiments (including the simulation evaluation), we tested our method in Ubuntu 20.04 running on a desktop equipped with an Intel i9-13900KF processor and 64 GB RAM. Since our evaluation baseline, BALM2 \cite{liu2023efficient}, may require substantial memory to store the Hessian matrix, another 128 GB swap space was implemented to avoid memory overflow during the experiment. For the distributed bundle adjustment experiment, we implemented our distributed mapping algorithm with ROS and SFTP communication framework and tested our code on four consumer-level laptops, equipped with different CPUs and sizes of RAM, and connected with a WLAN network, as illustrated in Figure \ref{fig.devices}.
\begin{figure} [ht]
	\centering
	\includegraphics[width=\linewidth]{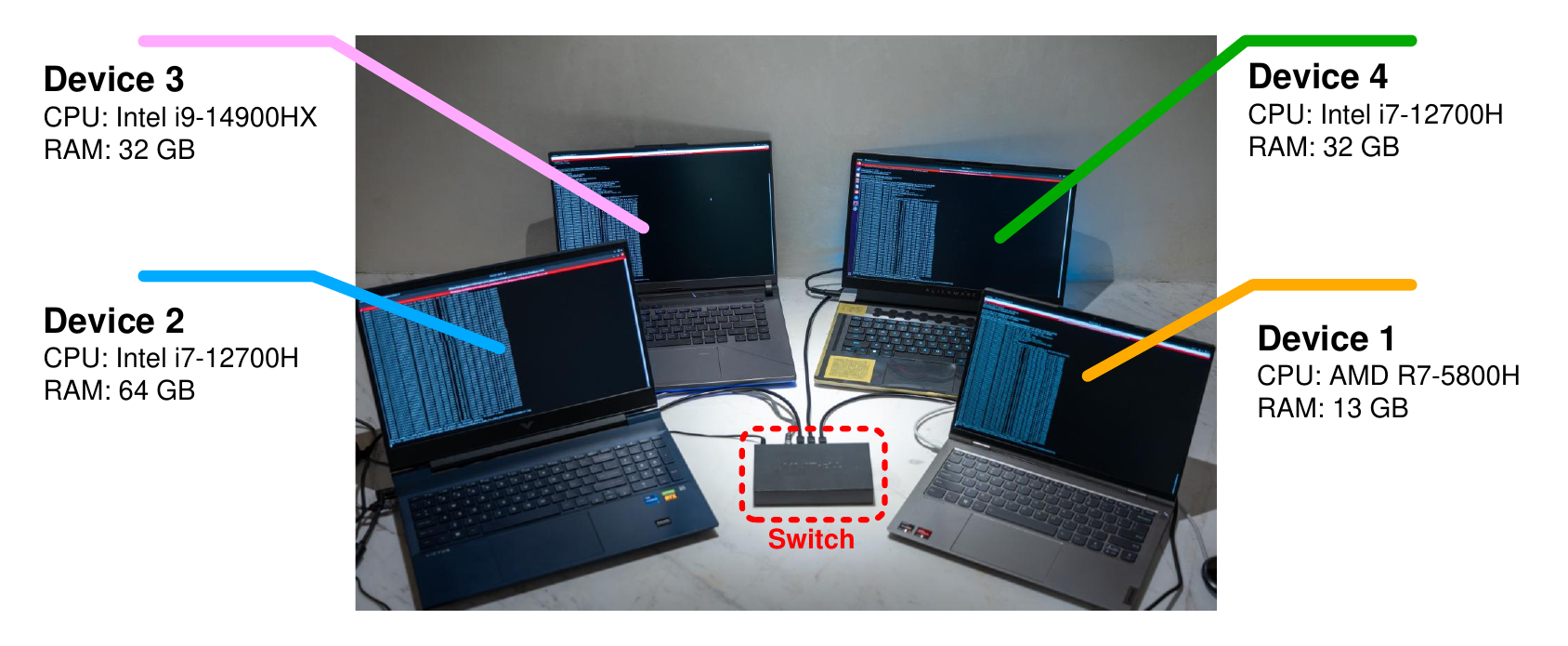}
	\caption{Devices used in distributed bundle adjustment experiment.}
	\label{fig.devices}
\end{figure}

The code implementation details of our optimizer are generally kept the same as \cite{liu2023efficient} except for the linear solver. As the Hessian matrix of the proposed method is block diagonal, the linear equation can be separated into a set of sub-problems. Thus, for the linear solver in the Levenberg-Marquardt optimizer, we utilize the LDLT Cholesky decomposition method to resolve each decoupled sub-problem in (\ref{eq.seperated_linear_equation}). These diagonal block solvers are computed parallel using the \textit{thread-pool} library in C++. The parallel thread number is determined by the available computation core in the CPU (e.g. 32 threads for Intel 13900KF used in the single-device experiment).

\section{Simulation Evaluation} \label{simulation}
\begin{figure} [ht]
	\centering
	\includegraphics[width=\linewidth]{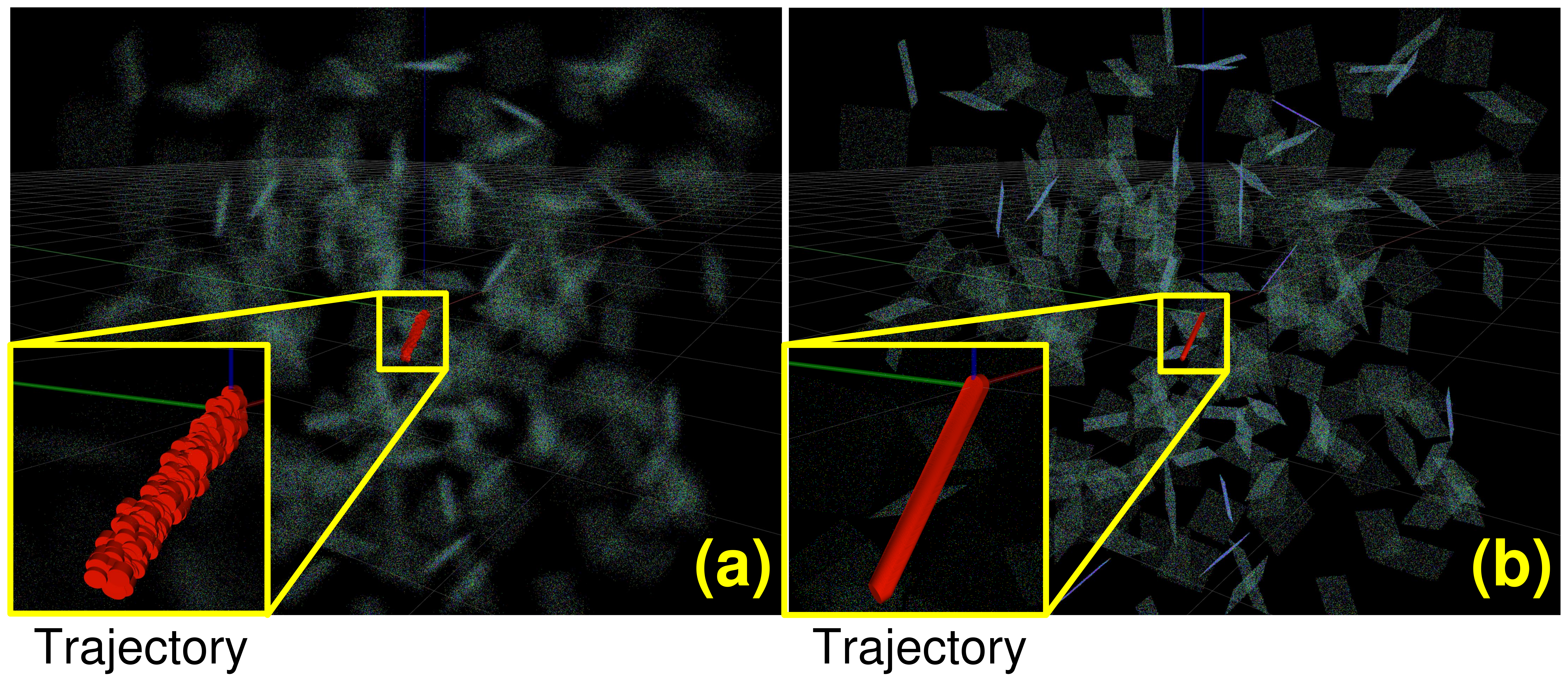}
	\caption{An example of our simulation environment. (a) shows the initial poses and point clouds. (b) shows the optimized one.}
	\label{fig.simulation_env}
\end{figure}
First, we verify the proposed method's convergence, accuracy, and efficiency under an ideal simulation environment. The simulation environment is the same as BALM2 \cite{liu2023efficient}. During the experiment, 200 planes with random centers and normal vectors are generated within a $10m\times 10m\times 10m$ cube space. LiDAR scans are then generated by randomly generating 5 points on each plane for each scan. Then, a random perturbation is added to each scan pose. An example of our simulation environment is shown in Figure \ref{fig.simulation_env}.

To extensively analyze the time complexity of the proposed method and compare it with BALM2, we evaluated BALM2 and the proposed method under different scan numbers from $2^7$ {(i.e., 128)} to $2^{13}$ {(i.e., 8,192)} and repeated the experiment 5 times. Since the time and memory cost of BALM2 when optimizing 8,192 scan poses is extremely large, we did not evaluate it. During optimization, the terminating threshold $\epsilon$ in Algorithm \ref{algo-LM} is set to $10^{-5}$ for both BALM2 and the proposed method.

\subsection{Convergence Evaluation} \label{sec.convergence_evaluation}
We first verify the convergence of the proposed method by comparing the optimized point-to-plane residual, as described in (\ref{BA-formulation-reduced-reduced}).

According to the theory of the majorization-minimization method, our optimization process should converge to the same global minimum residual as directly optimizing the original cost function. Furthermore, since the time complexity of our method has been reduced from cubic to linear, it is expected to demonstrate a significantly faster convergence speed. To verify this, we evaluated and plotted the optimization residual over time, as shown in Figure \ref{fig.convergence_exp}.

\begin{figure} [ht]
	\centering
	\includegraphics[width=\linewidth]{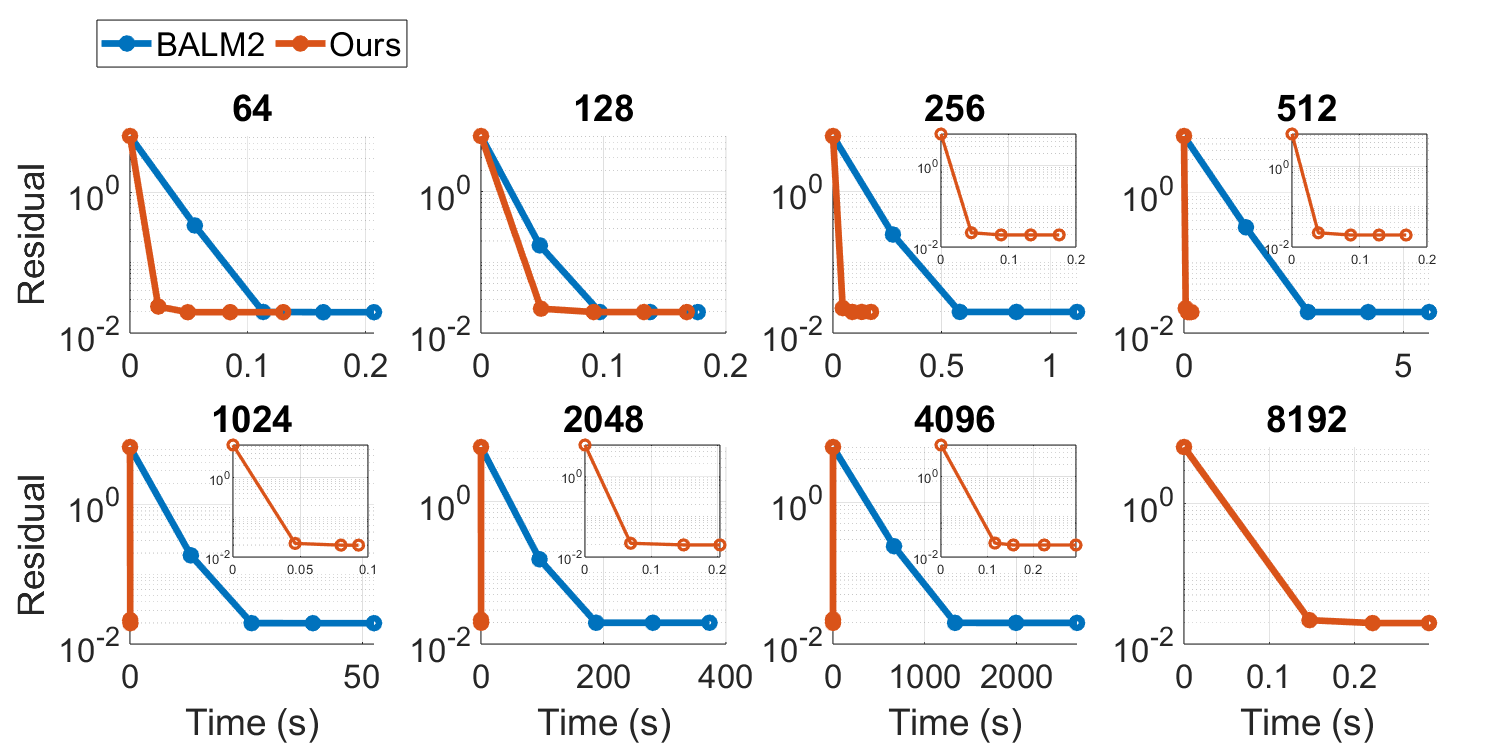}
	\caption{Result of the convergence evaluation under simulation environment. {The results demonstrate that the proposed method converges, at a significantly higher speed, to the same point-to-plane residual as BALM2.}}
	\label{fig.convergence_exp}
\end{figure}

The results confirm that our proposed method consistently converges to the same final residual as BALM2 (average: $9.931 \times 10^{-4} (m)$ for both BALM2 and the proposed method, with a difference of less than $1 \times 10^{-8}$), while requiring significantly less computation time. This advantage becomes even more pronounced as the number of optimized poses increases, aligning with our theoretical analysis.

\subsection{Accuracy Evaluation}
\begin{figure} [ht]
	\centering
	\includegraphics[width=\linewidth]{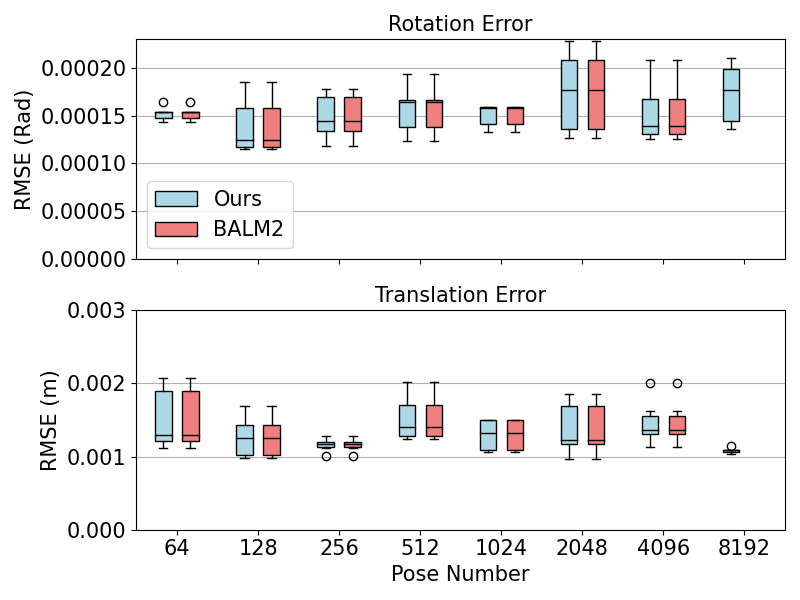}
	\caption{RMSE of rotation and translation error in simulation evaluation.}
	\label{fig.simulation_accuracy}
\end{figure}
We then evaluate and compare the accuracy of the proposed method and BALM2 by analyzing the Root Mean Square Error (RMSE) of the rotation and translation errors. Since the poses estimated by minimizing the point-to-plane residual are unbiased, as shown in \cite{liu2023efficient}, both BALM2 and the proposed method should achieve the same high accuracy when the point-to-plane residual converges to the same global minimum.

The evaluation results are presented in Figure \ref{fig.simulation_accuracy}. The results indicate that our method achieves the same level of accuracy as BALM2, which aligns with the theoretical expectations.

\subsection{Efficiency Evaluation}
\begin{figure} [ht]
	\centering
	\includegraphics[width=\linewidth]{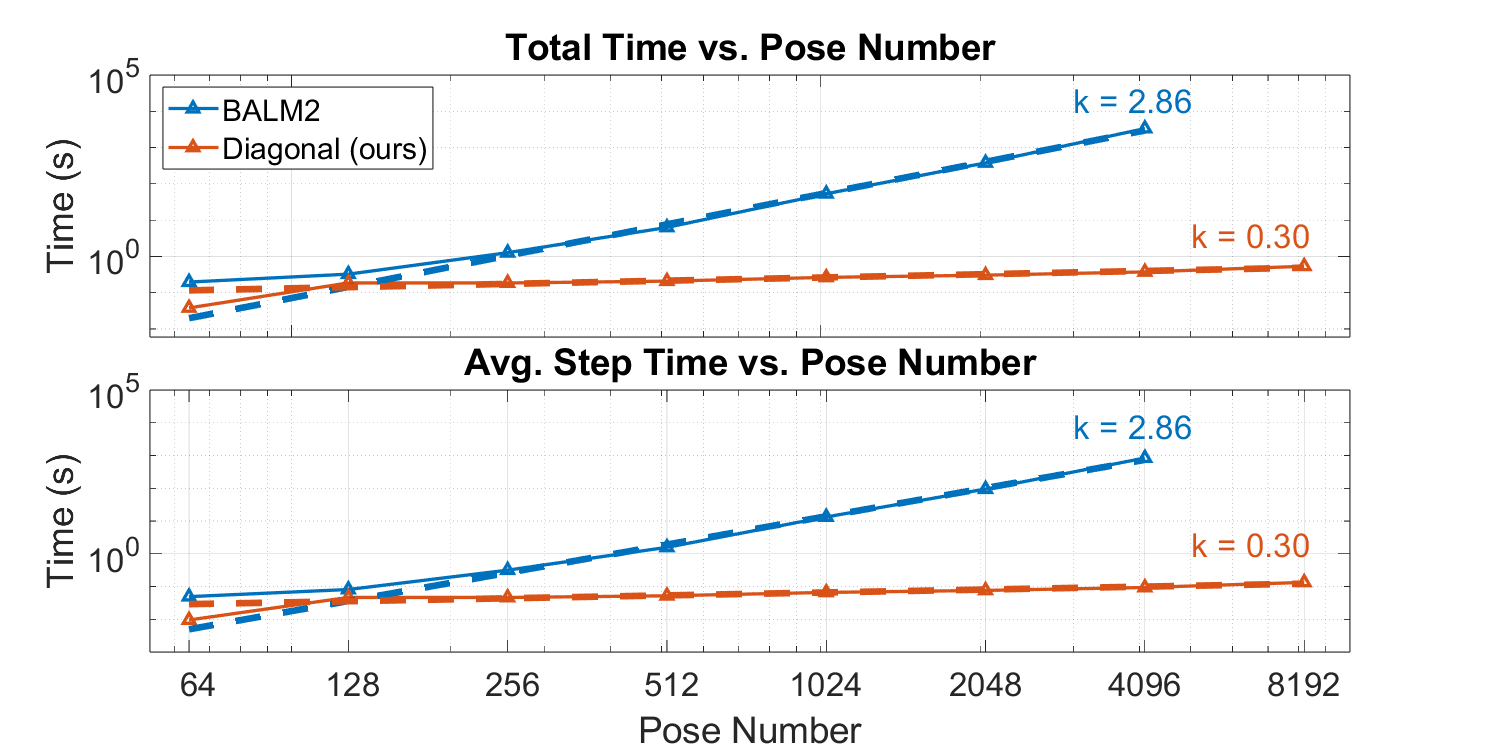}
	\caption{Log-log plot of the {optimization time} against pose number. $k$ expresses the slope of the first-order regression result of the log-log plot. In which case we have $\log(Time(s))=k\log(Pose Number)$.}
	\label{fig.time_cost_summary}
\end{figure}
We lastly compare the efficiency and the time complexity of the proposed method and BALM2 by comparing the optimization time. We summarized the {optimization time} $t$ against the scan pose number $M_p$ with a log-log plot, as shown in Figure \ref{fig.time_cost_summary}. A first-order regression is performed on the log-log plot. The slope $k$ of the regression result implies that
\begin{align}
    \log(t)=k \log(M_p)
\end{align}
and the time complexity of the optimization will be $O(M_p^{k})$. During regression, the optimization time of 64 and 128 poses is not considered, as the high-order terms may not dominate when the pose number is too small, leading to an inaccurate fitting result. The result implies that the time complexity of BALM2 is $O(M_p^{2.86})$, and the time complexity of the proposed method is $O(M_p^{0.30})$. It should be noted that the evaluated time complexity of the proposed method is even less than the theory value which is $O(M_p)$. This is because the time cost of the proposed method is still very small even when optimizing 8,192 poses (i.e. less than 0.5 second, refer to {Figure \ref{fig.convergence_exp}}), and the time for solving the linear equation in LM optimizer may still not dominate. Nevertheless, the lower time complexity of the proposed method is already able to be proved.

\section{Benchmark Evaluation} \label{benchmark}
In this section, extensive experiments are conducted on various public datasets with different kinds of environments, LiDAR types, and platforms to evaluate the accuracy and computational efficiency against the state-of-the-art point cloud bundle adjustment approach.
\subsection{Datasets and Pre-process} \label{sec.bench_preprocess}
We conducted our experiments on four public datasets: \textit{MulRan} \cite{mulran}, \textit{Hilti2022} \cite{hilti}, \textit{HeLiPr} \cite{jung2023helipr}, and \textit{MaRS-LVIG} \cite{marslvig}, each providing a variety of environments, LiDAR types, mounted platforms, and a wide range of pose numbers and map scales.

Among these datasets, the \textit{Hilti2022} dataset primarily features indoor and small-scale outdoor structured environments. The \textit{MulRan} and \textit{HeLiPr} datasets focus on large-scale urban environments. Since both datasets captured point cloud data from multiple LiDAR types, we selected the mechanical rotating Ouster LiDAR for \textit{MulRan} and the solid-state Avia LiDAR for \textit{HeLiPr} to evaluate the proposed method across different fields of view (FoV) and scan patterns. Additionally, the \textit{HeLiPr} dataset contains the longest time duration and trajectory length. The \textit{MaRS-LVIG} dataset is an airborne mapping dataset, offering unique view directions and a wide map range.

We pre-processed the sequences for evaluation and observed that some of them contain scenarios where the LiDAR scans are degenerated, which falls outside the scope of this paper. Consequently, we excluded those sequences from further evaluation. For all processed sequences, we first employed FAST-LIO2 \cite{xu2022fast}, a robust LiDAR-inertial odometry system, to compensate for motion distortion in the LiDAR scans and generate initial pose estimations. We then used BTC \cite{yuan2024btc}, an efficient and robust point cloud descriptor, and GTSAM \cite{gtsam} to perform loop closure on the initial estimations. Finally, the distorted LiDAR scans, along with their corresponding initial pose estimations, were optimized using both BALM2 and the proposed method. Detailed information on the processed sequences is provided in Table \ref{tab:dataset_summary}.

\subsection{Evaluation Setup}
Although the proposed method is highly efficient across all sequences without requiring any point cloud downsampling or keyframe extraction, we observed that BALM2 becomes excessively time-consuming as the number of poses increases. To address this, we extracted one keyframe for every five LiDAR scans and optimized only the poses of the keyframes when evaluating BALM2.
{For a consistent comparison and comprehensive evaluation, we present the optimization results of the proposed method using the same keyframes. Additionally, since our method can well handle large number of LiDAR poses without a problem, we also present the results of our method using all frames without keyframe selection. The two variants of our methods are referred to as {``KF''} and {``AF''}, respectively.}
Moreover, the scan number of sequences in the HeLiPR dataset is over 5000 even after keyframe extraction. When optimizing these sequences, BALM2 still requires a substantial optimization time (more than 2 hours for each iteration).
Thus, we did not evaluate BALM2 on the HeLiPR dataset.

During the evaluation, we observed that the public code of BALM2 contains significant redundant memory usage and lacks optimization for computation time in the feature extraction process (i.e., adaptive voxelization), which considerably limits the scale of the processed data. To address this, we optimized the code for this process. In the evaluation, both BALM2 and the proposed method utilized our optimized feature extraction code. For each dataset, all parameters (e.g., root voxel size, octo-tree layer number, planar eigenratio threshold, etc.) are kept the same to ensure that both methods use the same planar features and point clusters as input.

\begin{table}[]
\centering
\caption{The summary of the datasets}
\label{tab:dataset_summary}
\resizebox{\columnwidth}{!}{%
\begin{tabular}{@{}ccccc@{}}
\toprule
Dataset & Sequence & Pose Number & Length (m) & LiDAR Type \\ \midrule
\multirow{4}{*}{Hilti 2022} & exp 01 & 2273 & 165.80 & \multirow{4}{*}{Pandar 64} \\
 & exp 02 & 4298 & 327.73 &  \\
 & exp 07 & 1319 & 115.05 &  \\
 & exp21 & 1433 & 129.02 &  \\ \midrule
\multirow{9}{*}{Mulran} & DCC 01 & 5539 & 4912.21 & \multirow{9}{*}{Ouster 64} \\
 & DCC 02 & 7557 & 4273.45 &  \\
 & DCC 03 & 7474 & 5421.82 &  \\
 & KAIST 01 & 8222 & 6123.63 &  \\
 & KAIST 02 & 8938 & 5965.20 &  \\
 & KAIST 03 & 8624 & 6249.14 &  \\
 & Riverside 01 & 5533 & 6427.23 &  \\
 & Riverside 02 & 8154 & 6612.67 &  \\
 & Riverside 03 & 10473 & 7248.56 &  \\ \midrule
\multirow{9}{*}{HeLIPR} & Roundabout 01 & 27294 & 9041.73 & \multirow{9}{*}{Avia} \\
 & Roundabout 02 & 20845 & 7449.30 &  \\
 & Roundabout 03 & 25147 & 9262.55 &  \\
 & Town 01 & 24139 & 7830.79 &  \\
 & Town 02 & 26885 & 8203.48 &  \\
 & Town 03 & 25281 & 8899.70 &  \\
 & Bridge 01 & 21462 & 23056.71 &  \\
 & Bridge 02 & 25615 & 14611.60 &  \\
 & Bridge 03 & 20082 & 19397.79 &  \\ \midrule
\multirow{18}{*}{MaRS-LVIG} & AMtown 01 & 12761 & 4635.24 & \multirow{18}{*}{Avia} \\
 & AMtown 02 & 6299 & 4929.88 &  \\
 & AMtown 03 & 4934 & 4752.61 &  \\
 & AMvalley 01 & 10716 & 3999.13 &  \\
 & AMvalley 02 & 6461 & 4019.40 &  \\
 & AMvalley 03 & 4646 & 4011.83 &  \\
 & HKairport 01 & 7248 & 1880.31 &  \\
 & HKairport 02 & 3788 & 1885.10 &  \\
 & HKairport 03 & 3168 & 1891.93 &  \\
 & HKairport GNSS 01 & 7094 & 1887.25 &  \\
 & HKairport GNSS 02 & 3801 & 1953.94 &  \\
 & HKairport GNSS 03 & 3119 & 1911.81 &  \\
 & HKisland 01 & 6593 & 1686.67 &  \\
 & HKisland 02 & 3780 & 1697.54 &  \\
 & HKisland 03 & 2876 & 1690.01 &  \\
 & HKisland GNSS 01 & 6853 & 1734.86 &  \\
 & HKisland GNSS 02 & 3811 & 1724.86 &  \\
 & HKisland GNSS 03 & 3035 & 1717.70 &  \\ \bottomrule
\end{tabular}%
}
\end{table}

\begin{table*}[]
\centering
\caption{The evaluation result of our benchmark experiment.}
\label{tab:benchmark_result}
\resizebox{\textwidth}{!}{%
\begin{tabular}{@{}cccccccccccc@{}}
\toprule
\multirow{2}{*}{Dataset} & \multirow{2}{*}{Sequence} & \multicolumn{4}{c}{RMSE of APE (m)} & \multicolumn{3}{c}{Time consumption (s)} & \multicolumn{3}{c}{Peak Memory usage (GB)} \\ \cmidrule(l){3-6} \cmidrule(l){7-9} \cmidrule(l){10-12} 
 &  & FAST-LIO & BALM2 & Ours ({KF$^1$}) & Ours ({AF$^2$}) & BALM2 & Ours ({KF}) & Ours ({AF}) & BALM2 & Ours ({KF}) & Ours ({AF}) \\ \midrule
\multirow{4}{*}{Hilti 2022} & exp 01 & 0.033 & \textbf{0.032} & \textbf{0.032} & 0.031 & 4.77 & \textbf{1.33} & 3.48 & 1.5 & \textbf{1.0} & 4.2 \\
 & exp 02 & 0.034 & \textbf{0.030} & \textbf{0.030} & 0.029 & 76.14 & \textbf{5.06} & 5.26 & 3.1 & \textbf{1.4} & 6.1 \\
 & exp 07 & 0.060 & \textbf{0.054} & 0.056 & 0.059 & \textbf{2.33} & 3.89 & 0.96 & 0.5 & \textbf{0.3} & 1.2 \\
 & exp 21 & \textbf{0.044} & 0.046 & 0.045 & 0.045 & \textbf{1.90} & 2.35 & 2.44 & 0.8 & \textbf{0.6} & 2.5 \\ \midrule
\multirow{9}{*}{MulRan} & DCC01 & 5.467 & 5.718 & \textbf{5.342} & 5.352 & 96.93 & \textbf{2.67} & 6.22 & 4.8 & \textbf{1.4} & 6.6 \\
 & DCC02 & 2.998 & 3.325 & \textbf{2.907} & 2.932 & 731.23 & \textbf{2.92} & 7.35 & 8.0 & \textbf{1.9} & 9.0 \\
 & DCC03 & 2.196 & 2.598 & \textbf{2.074} & 2.076 & 396.32 & \textbf{3.03} & 8.15 & 7.9 & \textbf{1.9} & 9.0 \\
 & KAIST01 & 3.371 & 3.396 & \textbf{3.354} & 3.338 & 2312.43 & \textbf{4.52} & 8.04 & 9.1 & \textbf{1.9} & 9.0 \\
 & KAIST02 & 3.421 & 3.937 & \textbf{3.409} & 3.411 & 285.54 & \textbf{3.13} & 7.87 & 10.5 & \textbf{2.1} & 9.9 \\
 & KAIST03 & 2.943 & 3.217 & \textbf{2.937} & 2.932 & 305.50 & \textbf{3.03} & 8.15 & 10.0 & \textbf{2.1} & 9.9 \\
 & Riverside01 & 8.435 & 8.450 & \textbf{8.420} & 8.421 & 150.93 & \textbf{2.62} & 5.45 & 4.4 & \textbf{1.1} & 5.1 \\
 & Riverside02 & 11.434 & 11.443 & \textbf{11.424} & 11.439 & 524.94 & \textbf{2.73} & 8.30 & 8.6 & \textbf{1.7} & 7.7 \\
 & Riverside03 & 13.084 & \textbf{12.949} & 13.065 & 13.071 & 1015.64 & \textbf{3.25} & 10.25 & 13.1 & \textbf{2.0} & 9.7 \\ \midrule
\multirow{9}{*}{HeLiPR} & Roundabout 01 & 2.085 & --$^3$ & \textbf{1.665} & 1.722 & -- & \textbf{14.04} & 152.04 & -- & \textbf{3.1} & 10.5 \\
 & Roundabout 02 & 1.849 & -- & \textbf{1.539} & 1.471 & -- & \textbf{7.94} & 61.42 & -- & \textbf{2.4} & 12.6 \\
 & Roundabout 03 & 2.493 & -- & \textbf{1.971} & 1.870 & -- & \textbf{14.76} & 185.43 & -- & \textbf{2.9} & 9.8 \\
 & Town 01 & \textbf{4.462} & -- & 4.645 & 4.438 & -- & \textbf{5.92} & 14.60 & -- & \textbf{2.7} & 14.8 \\
 & Town 02 & 3.515 & -- & \textbf{3.357} & 3.325 & -- & \textbf{5.37} & 69.71 & -- & \textbf{3.0} & 11.0 \\
 & Town 03 & 4.688 & -- & \textbf{4.514} & 4.596 & -- & \textbf{7.49} & 27.42 & -- & \textbf{2.8} & 13.5 \\
 & Bridge 01 & \textbf{10.022} & -- & 11.823 & 11.186 & -- & \textbf{4.33} & 52.29 & -- & \textbf{2.3} & 12.7 \\
 & Bridge 02 & \textbf{6.423} & -- & 6.440 & 6.440 & -- & \textbf{4.05} & 19.78 & -- & \textbf{2.7} & 14.4 \\
 & Bridge 03 & \textbf{10.461} & -- & 10.469 & 10.461 & -- & \textbf{4.08} & 22.15 & -- & \textbf{2.1} & 13.3 \\ \midrule
\multirow{18}{*}{MaRS-LVIG} & AMtown 01 & 1.720 & \textbf{1.113} & 1.508 & 1.497 & 3857.27 & \textbf{5.48} & 23.04 & 16.8 & \textbf{2.0} & 9.0 \\
 & AMtown 02 & 6.811 & \textbf{6.569} & 6.777 & 8.544 & 668.73 & \textbf{3.98} & 13.79 & 4.9 & \textbf{1.1} & 4.6 \\
 & AMtown 03 & 4.733 & \textbf{4.103} & 4.615 & 4.661 & 189.82 & \textbf{4.54} & 15.55 & 3.2 & \textbf{0.9} & 3.7 \\
 & AMvalley 01 & 6.134 & 6.902 & \textbf{6.091} & 6.099 & 2747.08 & \textbf{4.98} & 19.52 & 12.9 & \textbf{1.8} & 8.0 \\
 & AMvalley 02 & 14.594 & \textbf{14.499} & 14.598 & 14.619 & 1272.11 & \textbf{2.78} & 7.33 & 5.1 & \textbf{1.1} & 4.5 \\
 & AMvalley 03 & 17.746 & \textbf{17.515} & 17.709 & 17.684 & 341.73 & \textbf{1.79} & 15.30 & 2.9 & \textbf{0.8} & 3.2 \\
 & HKairport GNSS 01 & 0.495 & \textbf{0.463} & 0.479 & 0.462 & 133.61 & \textbf{1.60} & 5.55 & 5.9 & \textbf{1.2} & 5.3 \\
 & HKairport GNSS 02 & \textbf{0.850} & 0.860 & 0.873 & 0.682 & 18.28 & \textbf{1.45} & 4.50 & 2.0 & \textbf{0.7} & 2.8 \\
 & HKairport GNSS 03 & 0.752 & \textbf{0.648} & 0.671 & 0.548 & 9.81 & \textbf{1.98} & 5.00 & 1.4 & \textbf{0.6} & 2.3 \\
 & HKairport 01 & 0.740 & \textbf{0.611} & 0.693 & 0.649 & 186.75 & \textbf{2.22} & 7.72 & 5.9 & \textbf{1.2} & 5.3 \\
 & HKairport 02 & 0.720 & \textbf{0.615} & 0.642 & 0.610 & 26.84 & \textbf{1.82} & 5.19 & 1.9 & \textbf{0.7} & 2.8 \\
 & HKairport 03 & 0.956 & \textbf{0.794} & 0.886 & 0.815 & 20.02 & \textbf{1.26} & 1.72 & 1.6 & \textbf{0.6} & 2.4 \\
 & HKisland GNSS 01 & \textbf{0.671} & 0.790 & 0.725 & 0.682 & 173.24 & \textbf{0.79} & 2.71 & 4.9 & \textbf{0.8} & 3.4 \\
 & HKisland GNSS 02 & 0.853 & \textbf{0.767} & 0.838 & 0.814 & 23.74 & \textbf{1.50} & 2.52 & 1.7 & \textbf{0.5} & 1.8 \\
 & HKisland GNSS 03 & 1.180 & \textbf{0.912} & 1.128 & 1.107 & 38.12 & \textbf{2.04} & 3.01 & 1.3 & \textbf{0.4} & 1.7 \\
 & HKisland 01 & \textbf{0.580} & 0.832 & 0.603 & 0.593 & 157.16 & \textbf{1.11} & 2.95 & 4.6 & \textbf{0.7} & 3.1 \\
 & HKisland 02 & \textbf{0.902} & 1.000 & 0.904 & 0.734 & 22.50 & \textbf{1.19} & 2.16 & 1.6 & \textbf{0.5} & 1.7 \\
 & HKisland 03 & 1.213 & 1.220 & \textbf{1.201} & 0.952 & 7.52 & \textbf{1.79} & 2.68 & 1.0 & \textbf{0.4} & 1.3 \\ \bottomrule
\end{tabular}%
}
\begin{tablenotes}
    \footnotesize
    \item $^1$ {KF\quad: Using the selected keyframes.}
    \item $^2$ {AF\quad: Using all frames without keyframe selection.}
    \item $^3$ --\quad: Unable to evaluate since BALM2 takes too long time (more than $2$ hours for each iteration).
\end{tablenotes}
\end{table*}

\subsection{Accuracy Evaluation} \label{sec.benchmark_acc}
The accuracy evaluation results are presented in Table \ref{tab:benchmark_result}. These results demonstrate that our proposed method achieves accuracy comparable to the baseline, BALM2.
Specifically, on the \textit{Hilti 2022} dataset, the proposed method and BALM2 exhibit similar levels of accuracy. For the \textit{MulRan} dataset, the proposed method generally achieves higher accuracy. On the \textit{MaRS-LVIG} dataset, both the proposed method and BALM2 significantly improve accuracy, with BALM2 typically achieving a lower RMSE of APE.

We then analyze the discrepancy in the performance between our method and BALM2. According to Section \ref{sec.majorization_minimization}, our method should always converge to the same final residual as BALM2, leading to similar accuracy.
However, during the experiments, only the results on the \textit{Hilti 2022} dataset were as expected. For sequences in the \textit{MulRan} and \textit{MaRS-LVIG} datasets, we observed some bias: our method achieved higher accuracy overall on the MulRan dataset, while BALM2 generally reached higher accuracy on the MaRS-LVIG dataset.

\begin{table}[]
\centering
\caption{Residual evaluation result of the proposed method compared with BALM2}
\label{tab:optimized_residual}
\resizebox{\columnwidth}{!}{%
\begin{tabular}{@{}cccccc@{}}
\toprule
\multirow{2}{*}{Dataset}    & \multirow{2}{*}{Sequence} & \multicolumn{3}{c}{Point-to-plane   Residual (m)} & \multirow{2}{*}{{Relative} Error $E_r$ (\%)} \\ \cmidrule(lr){3-5}
 &                   & Init     & BALM2    & Ours       &              \\ \midrule
\multirow{4}{*}{Hilti 2022} & exp01                     & 0.00078        & 0.000756       & 0.00075629      & 0.011879556                         \\
 & exp02             & 0.000724 & 0.000682 & 0.00068204 & 0.061826989  \\
 & exp07             & 0.000696 & 0.000654 & 0.00065473 & 0.154802513  \\
 & exp21             & 0.001449 & 0.001363 & 0.00136321 & 0.013755859  \\ \midrule
\multirow{9}{*}{MulRan}     & DCC01                     & 0.6065         & 0.58269        & 0.58497472      & 0.392031973                         \\
 & DCC02             & 0.589903 & 0.572816 & 0.57502868 & 0.38636252   \\
 & DCC03             & 0.611325 & 0.58973  & 0.59166306 & 0.327742096  \\
 & KAIST01           & 0.552895 & 0.529425 & 0.53039478 & 0.183126797  \\
 & KAIST02           & 0.545412 & 0.528309 & 0.52965771 & 0.255216117  \\
 & KAIST03           & 0.538161 & 0.525821 & 0.52689106 & 0.203526621  \\
 & Riverside01       & 0.554597 & 0.534711 & 0.53158311 & -0.584943807 \\
 & Riverside02       & 0.5569   & 0.529806 & 0.53208306 & 0.429804867  \\
 & Riverside03       & 0.535059 & 0.513857 & 0.51143388 & -0.47150477  \\ \midrule
\multirow{18}{*}{MaRS-LVIG} & AMtown01                  & 0.030732       & 0.029953       & 0.03001535      & 0.209187397                         \\
 & AMtown02          & 0.033919 & 0.032684 & 0.03279262 & 0.331396901  \\
 & AMtown03          & 0.035416 & 0.033925 & 0.03407417 & 0.43979997   \\
 & AMvalley01        & 0.060865 & 0.052944 & 0.05340681 & 0.874544556  \\
 & AMvalley02        & 0.060987 & 0.056357 & 0.05655881 & 0.358187787  \\
 & AMvalley03        & 0.043917 & 0.04187  & 0.04288759 & 2.429783071  \\
 & HKairport\_GNSS01 & 0.040335 & 0.039798 & 0.03981438 & 0.041155416  \\
 & HKairport\_GNSS02 & 0.042663 & 0.041979 & 0.04200763 & 0.068509915  \\
 & HKairport\_GNSS03 & 0.044307 & 0.043319 & 0.04335596 & 0.084869441  \\
 & HKairport01       & 0.039989 & 0.039177 & 0.03922305 & 0.1167463    \\
 & HKairport02       & 0.039889 & 0.039211 & 0.03923612 & 0.065326795  \\
 & HKairport03       & 0.04363  & 0.042772 & 0.0428755  & 0.241974615  \\
 & HKisland\_GNSS01  & 0.082186 & 0.080418 & 0.08045554 & 0.046291902  \\
 & HKisland\_GNSS02  & 0.08168  & 0.079549 & 0.0796355  & 0.108736488  \\
 & HKisland\_GNSS03  & 0.083776 & 0.081833 & 0.0818415  & 0.010948383  \\
 & HKisland01        & 0.068848 & 0.067804 & 0.06787098 & 0.098494413  \\
 & HKisland02        & 0.069898 & 0.068869 & 0.06894176 & 0.104943793  \\
 & HKisland03        & 0.07444  & 0.073163 & 0.07324577 & 0.113079022  \\ \bottomrule
\end{tabular}%
}
\end{table}

To further analyze this scenario, we compared the optimized point-to-plane residuals of BALM2 and our method, and calculated the relative error $E_r$ as
\begin{align}
    E_r=\frac{r_{\text{ours}}-r_{\text{balm}}}{r_{\text{balm}}}\times100\%
\end{align}
where $r_{\text{balm}}$ and $r_{\text{ours}}$ represents the point-to-plane residual optimized by BALM2 and our method, respectively. The result is shown in Table \ref{tab:optimized_residual}.
The results indicate that both BALM2 and our method converge to the same optimal residual, with our method yielding a slightly higher final residual (an average of $0.22\%$ overall). Such behavior is also aligned with our expectations.
Nevertheless, given the physical meaning of the point-to-plane residual, this suggests that the optimized plane features from BALM2 and our method differ by only an extremely small thickness, which is imperceptible in the map. A visualization of the mapping results is shown in Figure \ref{fig:benchmark}. The results demonstrate that both BALM2 and the proposed method effectively improve map consistency and eliminate divergence.
Furthermore, the differences between the mapping results of BALM2 and our method are indistinguishable.
The accuracy bias on the MulRan and MaRS-LVIG datasets is likely caused by the errors introduced by the low-quality plane features.
More visualization results can be found in our video at \url{https://youtu.be/b8Onp0typMU}.
\begin{figure} [ht]
	\centering
	\includegraphics[width=\linewidth]{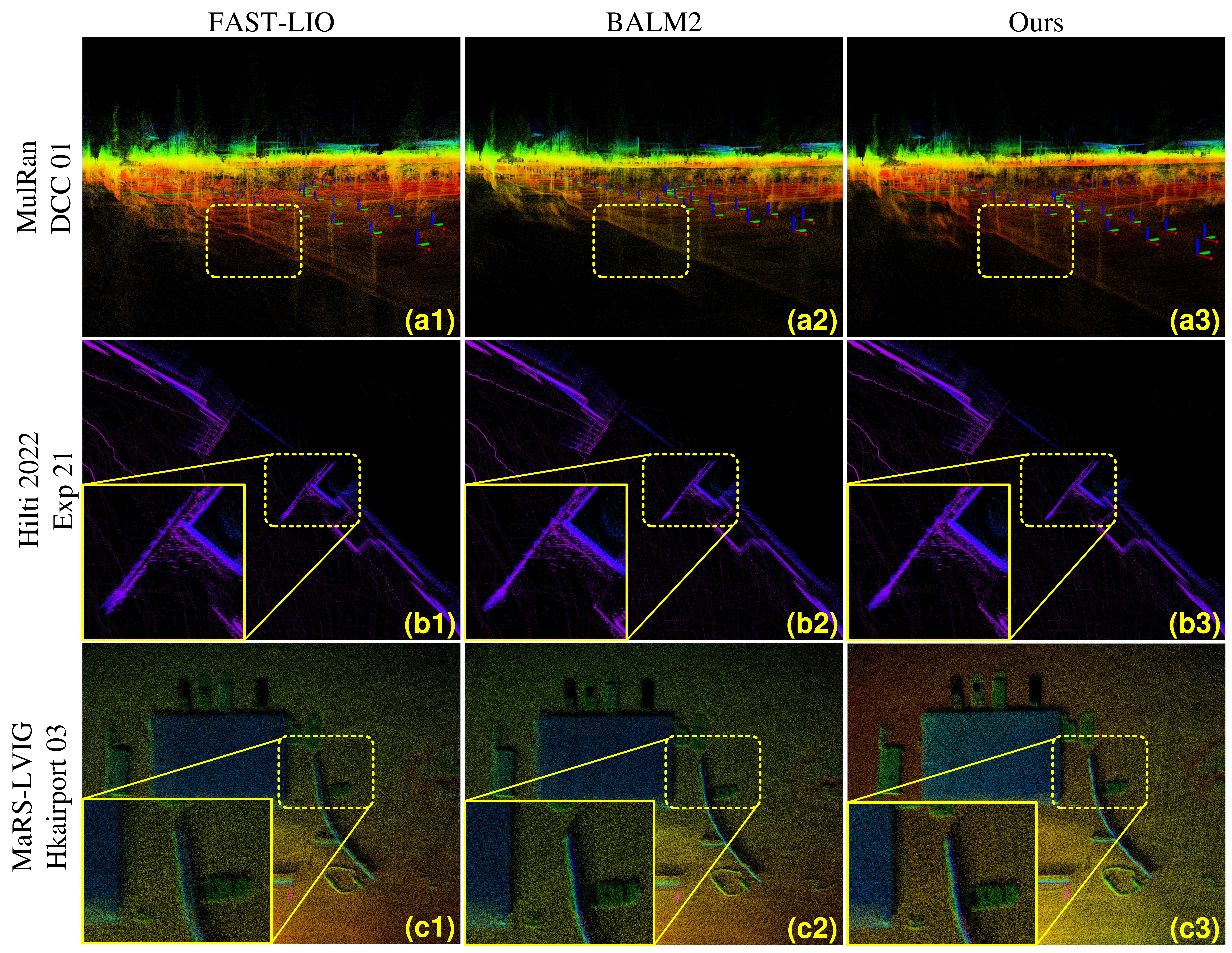}
	\caption{The result of our benchmark evaluation. (a1), (b1), and (c1) illustrates the mapping result of FAST-LIO. (a2), (b2), and (c2) shows the corresponding result of BALM2. (a3), (b3), and (c3) shows the result of our proposed methods. As shown in the figures, our method could effectively optimize the map consistency, reaching a similar result compared with BALM2.}
	\label{fig:benchmark}
\end{figure}
\subsection{Computational Efficiency Evaluation}
The efficiency evaluation results are also reported in Table \ref{tab:benchmark_result}. In summary, the results demonstrate that our method achieves significant improvements in both time and memory efficiency. 

\subsubsection{Time Efficiency}
In terms of time efficiency, our method outperforms the baseline on almost all sequences. Notably, the improvement becomes more pronounced with longer sequences, with the most significant enhancement observed on \textit{AMtown} 01, achieving a 703-fold efficiency increase. The only exceptions occur on the exp07 and exp21 sequences from the Hilti 2022 dataset. These sequences contain only 263 and 286 poses after keyframe extraction, making them the shortest sequences in our experimental datasets. As a result, the improvements in solving linear equations and evaluating the Hessian matrix are less pronounced.

\begin{figure*} [ht]
	\centering
	\includegraphics[width=0.85\linewidth]{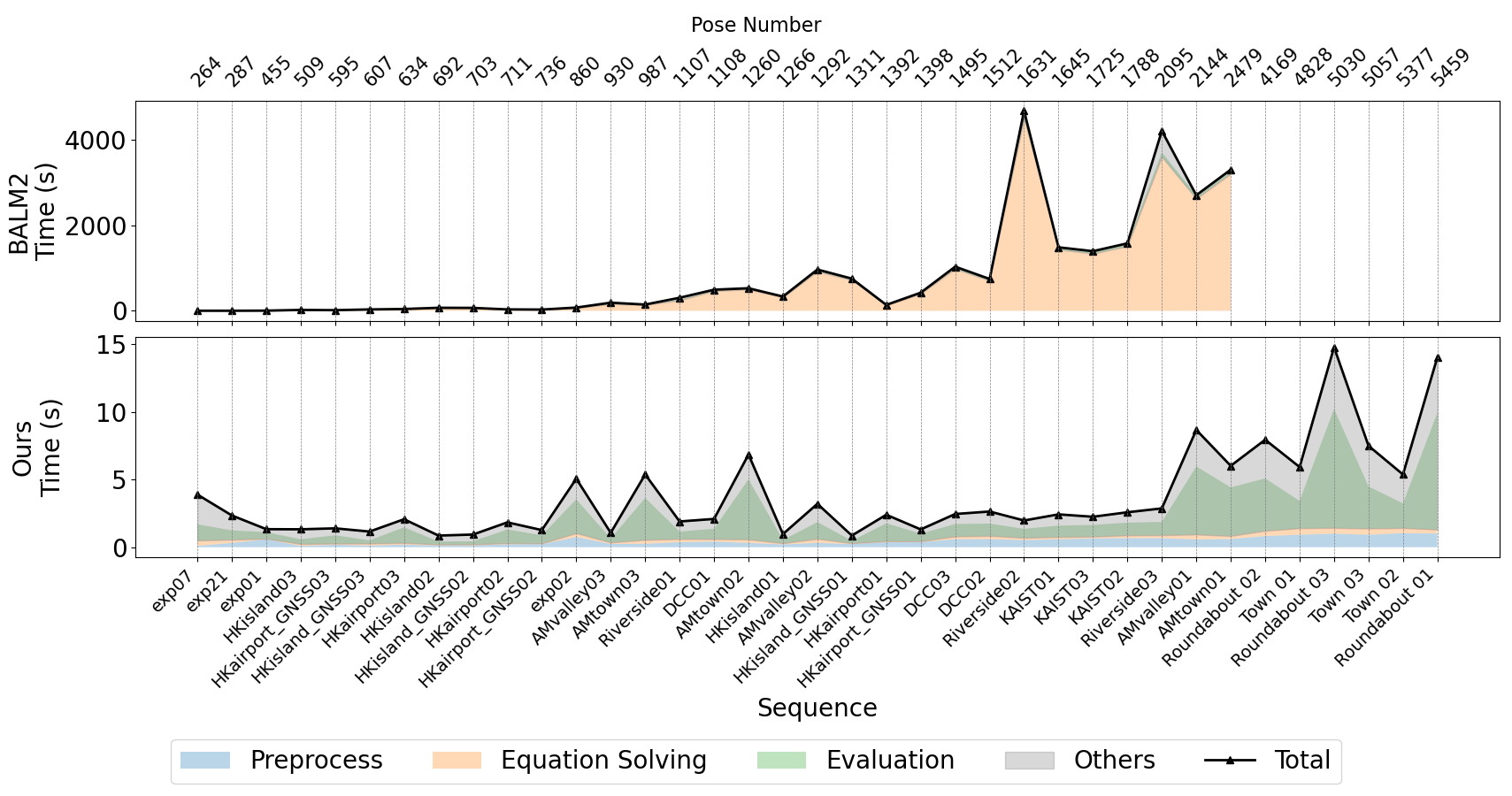}
	\caption{Detailed time consumption evaluation based on benchmark experiment results. The upper figure shows the evaluation result of BALM2, while the lower figure displays the result of the proposed method. {Notice the scale difference on the Y axis.}}
	\label{fig:time_consumption_detailed}
\end{figure*}

We also provide a more detailed evaluation of time consumption, as shown in Figure \ref{fig:time_consumption_detailed}. In this evaluation, the time consumption of the entire process is categorized into four groups: pre-processing, linear equation solving, derivative evaluation, and others (which include all processes not explicitly listed).
Compared to BALM2, the time required to solve linear equations is significantly lower with the proposed method (0.17 seconds on average for the proposed method versus 779.71 seconds on average for BALM2). This represents the primary source of efficiency improvement and aligns with our theoretical analysis and simulation experiment results.
Additionally, the time required to evaluate the Jacobian and Hessian matrices is also significantly reduced (1.18 seconds on average for the proposed method versus 32.73 seconds on average for BALM2). This improvement is due to the proposed method involving fewer computations than BALM2, despite having the same time complexity, and requiring less time to allocate memory for storing the Hessian matrix.

\subsubsection{Memory Efficiency}
In terms of memory efficiency, our method consistently outperforms BALM2 across all sequences. Similar to the improvements in time efficiency, the memory efficiency of our method becomes increasingly significant as the number of poses grows, due to its lower space complexity compared to BALM2. The most substantial improvement is observed in the \textit{AMtown 01} sequence, where our method requires only 2.0 GB of memory, compared to the 16.8 GB used by BALM2.
\begin{figure} [ht]
	\centering
	\includegraphics[width=\linewidth]{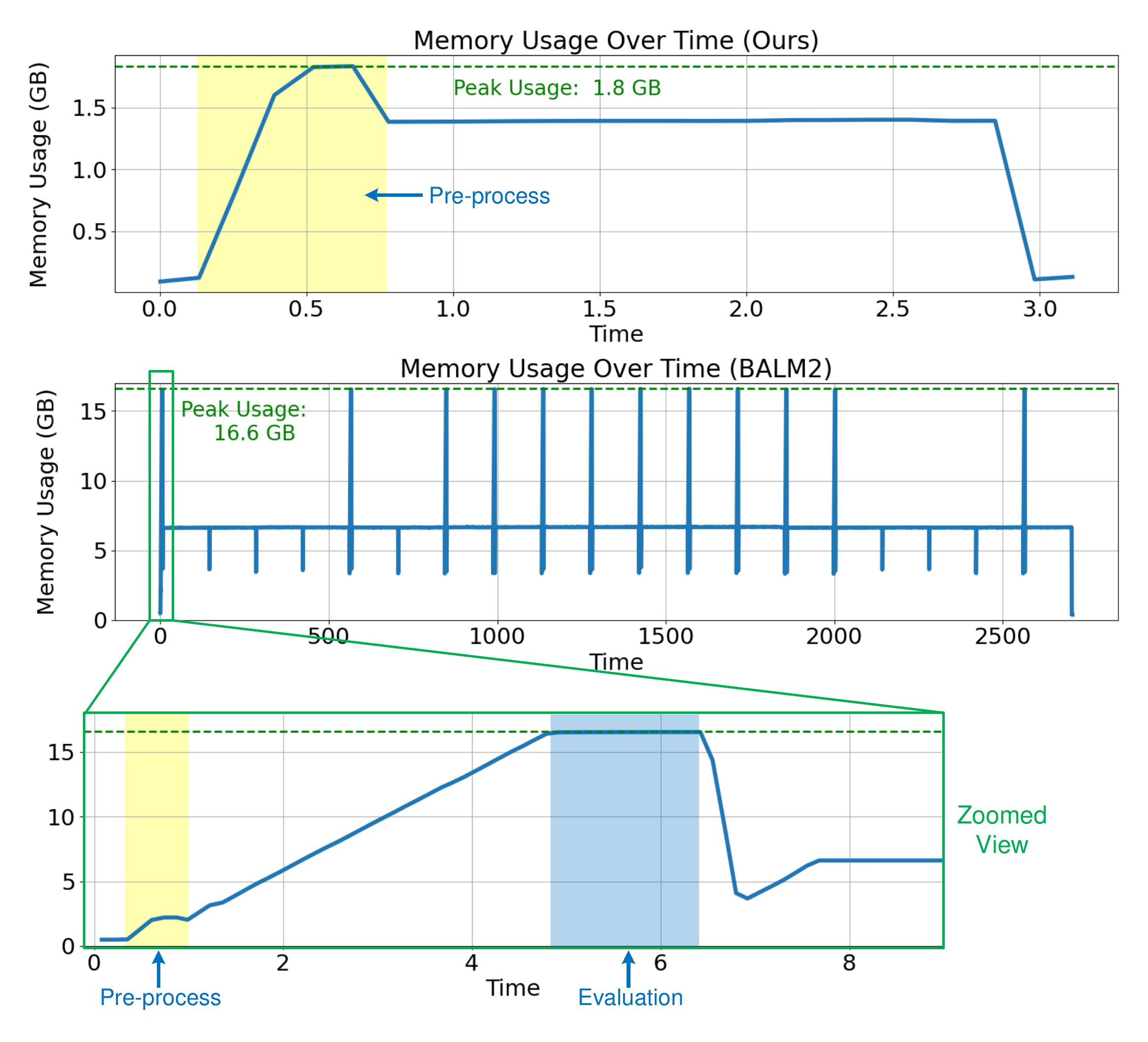}
	\caption{Evaluation result of memory usage over time during benchmark evaluation. {Notice the scale difference on the Y axis.}}
	\label{fig.memory_log}
\end{figure}

During the evaluation, we observed that the peak memory usage of the proposed method and BALM2 occurred at different stages. To illustrate this behavior, we analyzed and presented the memory usage over time for both BALM2 and the proposed method using the \textit{AMtown 01} sequence, which consists of 2,479 poses and 1.6 GB of point cloud data, as shown in Figure \ref{fig.memory_log}.

For BALM2, the peak memory usage occurred during the Hessian evaluation process, which has a space complexity of $O(M_p^2)$. At this stage, significant memory is required to store the dense Hessian matrix. Specifically, when processing the \textit{AMtown 01} sequence, storing the Hessian matrix during evaluation requires $(2479 \times 6)^2$ floating-point numbers, amounting to approximately 1.65 GB of memory per Hessian matrix. Furthermore, since the evaluation is performed using multiple threads, BALM2 maintains one copy of the Hessian matrix for each thread to ensure thread safety. During the experiment, with 8 threads in use, the total memory consumption reached approximately 13.18 GB—exceeding the memory required to store the raw point cloud data itself.

In contrast, the proposed method, which benefits from a space complexity of $O(M_p)$, demonstrated significantly lower memory usage during the evaluation process. The peak memory usage for the proposed method, measured at 1.8 GB, occurred during the adaptive voxelization process, where all raw points are temporarily stored in memory. After the adaptive voxelization step, the point clusters are computed, and the raw points are discarded, reducing memory usage to approximately 1.4 GB. Consequently, the peak memory usage of the proposed method is only slightly larger than the size of the point cloud data.

Moreover, this observation also suggests that the memory bottleneck in the proposed method lies within the feature extraction and matching process (i.e., adaptive voxelization), rather than in the optimization process itself.
% By optimizing this stage, the proposed method could handle even larger datasets.

\section{Distributed Mapping Experiment} \label{sec.distributed_mapping_experiment}
In this section, we reported our multiple-device distributed mapping experiment for the proposed distributed bundle adjustment framework introduced in Section \ref{sec.distributed_framework}.
\subsection{Experiment Setup}
We tested our framework on the sequence \textit{Bridge 01} from the \textit{HeLiPr} dataset. In this section, we choose the point cloud data captured using an OS2-128\footnote{\url{https://ouster.com/products/hardware/os2-lidar-sensor}}, a 128-channel spinning LiDAR. This LiDAR can provide over 2 million accurately measured points per second. Consequently, the resulting point cloud is extremely dense, requiring significant memory for bundle adjustment processing. In the tested sequence, the LiDAR generated 70.1 GB of point cloud data, which considerably exceeds the capacity of a single device. Thus, we distributed the data across four laptops and performed the 
distributed bundle adjustment process.

We first pre-processed the data. Similar to Section \ref{sec.bench_preprocess}, we used FAST-LIO2 \cite{xu2022fast} for motion compensation of each scan and to generate initial pose estimates. Additionally, we employed BTC \cite{yuan2024btc} for loop closure. The processed LiDAR scans were then divided into four sets based on the initial position estimates, and each set was distributed to a laptop illustrated in Figure \ref{fig.devices}. The separated point clouds and their trajectories are illustrated in Figure \ref{fig.reigon_seperation}.
\begin{figure} [ht]
	\centering
	\includegraphics[width=\linewidth]{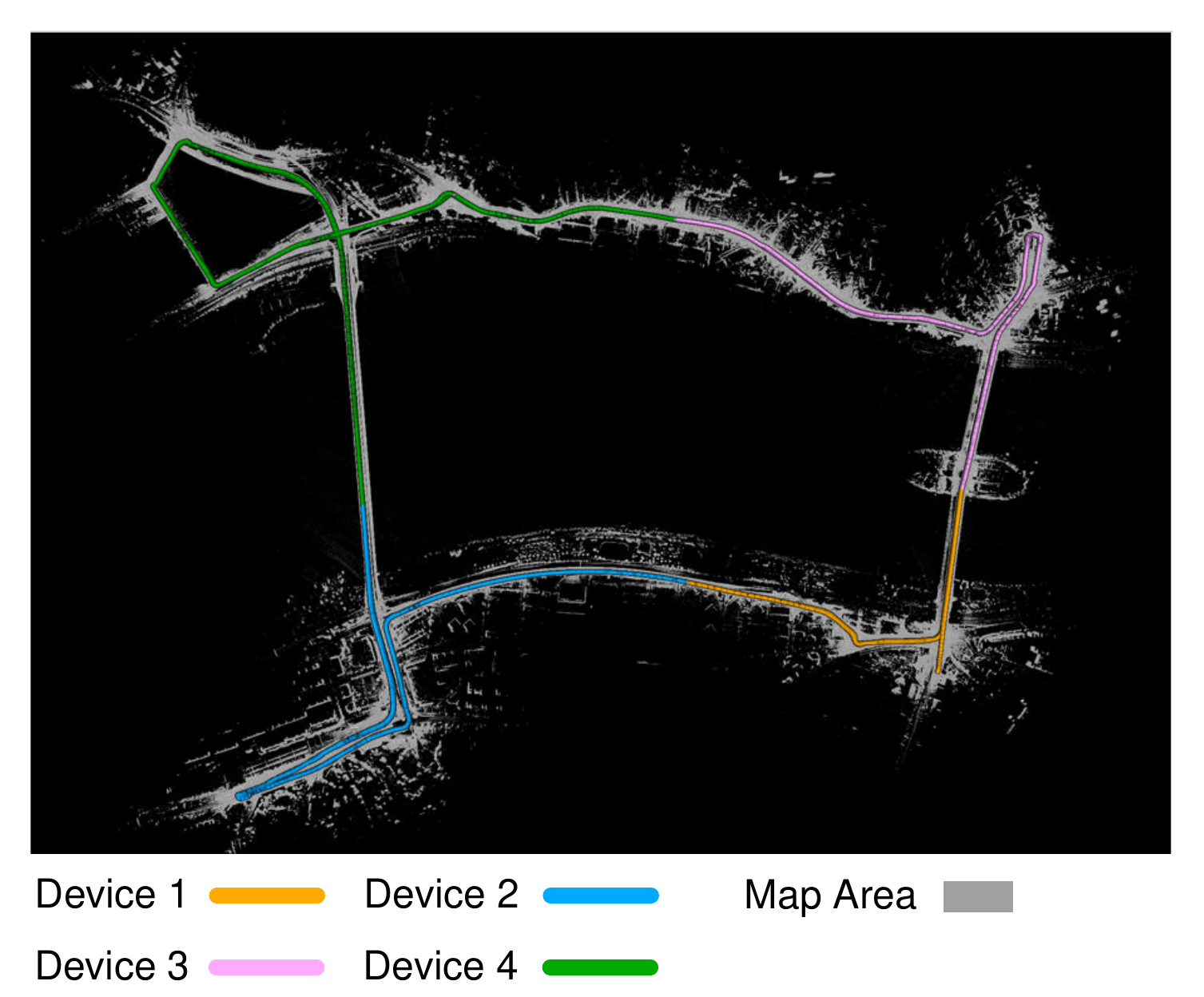}
	\caption{Results of sequence data separation and distribution: The sequence data from Bridge01 was divided into four sets, each of which was handled and processed by a separate computational device.}
	\label{fig.reigon_seperation}
\end{figure}
And more details of the data in each set are illustrated in the TABLE. \ref{tab:distributed_data}.
\begin{table}[ht]
\centering
\caption{Details of the distributed data.}
\label{tab:distributed_data}
% \resizebox{\columnwidth}{!}
{%
\begin{tabular}{@{}ccc@{}}
\toprule
Device   ID & Pose Number & Point Cloud Size (GB) \\ \midrule
1           & 3348       & 10.3                   \\
2           & 5519       & 18.7                  \\
3           & 6421        & 21.4                   \\
4           & 6148       & 19.7                    \\ \midrule
Total       & 21436      & 70.1                \\ \bottomrule
\end{tabular}%
}
\end{table}

\subsection{Experiment Results and Analysis}
We then performed the optimization process and reported the evaluation results as below:
\subsubsection{Map Consistency}
\begin{figure} [ht]
	\centering
	\includegraphics[width=\linewidth]{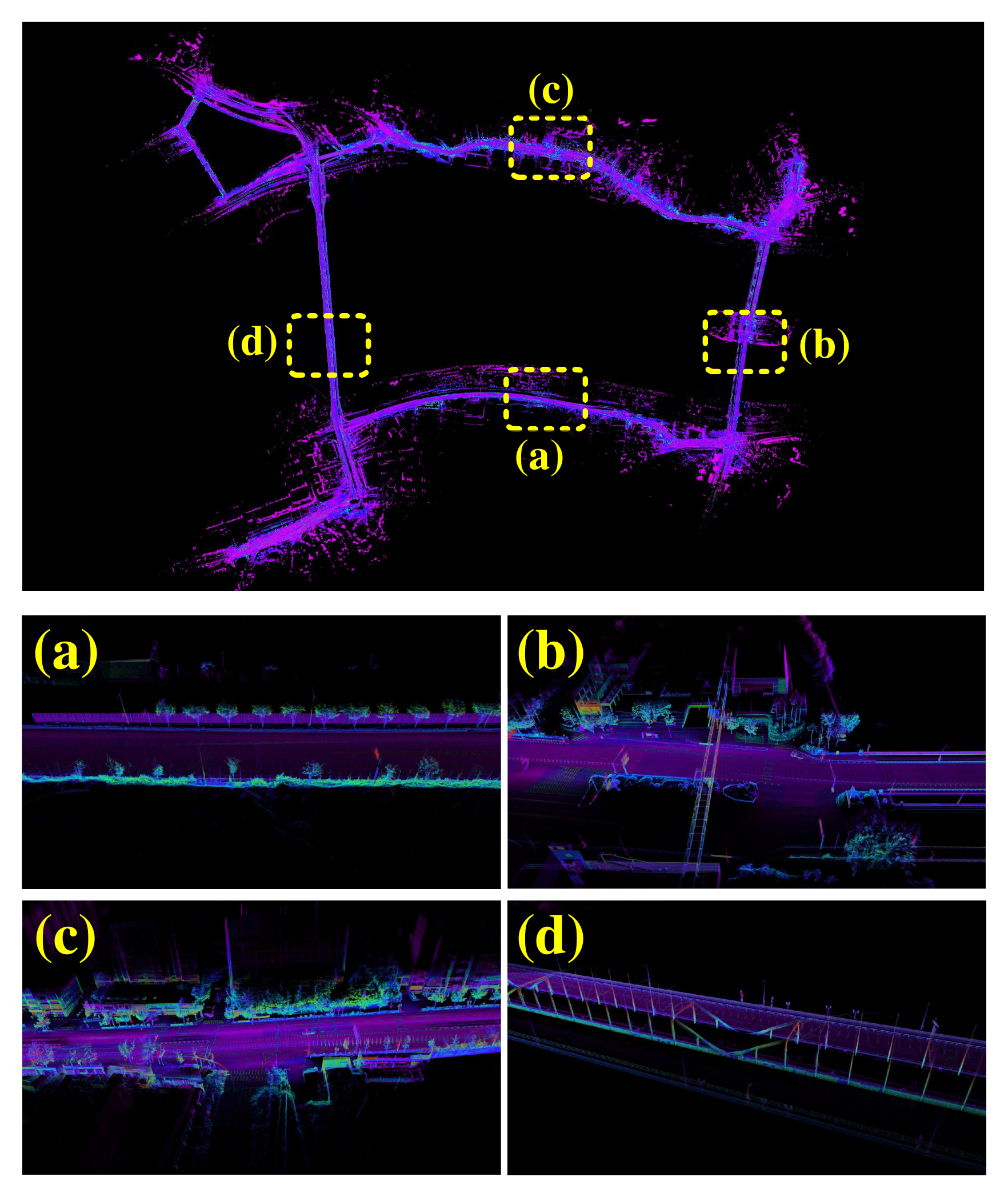}
	\caption{Mapping results of the proposed distributed bundle adjustment. The top figure provides an overview of the optimized map, while (a) to (d) display detailed views of the overlapping areas between the maps on each device.}
	\label{fig.distributed_pc}
\end{figure}

We first evaluated the map consistency of the proposed method. A visualization of the results, including an overview of the map and detailed views of four overlapping map areas from different devices, is presented in Figure \ref{fig.distributed_pc}. The results demonstrate that the proposed method successfully optimized the map and ensured the consistency of maps between devices. Detailed features such as lane markings, signboards, and street trees are also visible and clear in the visualization, making the map suitable for applications like autonomous driving and transportation.

\subsubsection{Time Efficiency Evaluation}
We then evaluated the time efficiency during the experiment. The time consumption, along with detailed information, is reported in Figure \ref{fig.distributed_time_consumption}. During the evaluation, we categorized the time consumption into four primary processes: the data loading and fixed-resolution voxelization process, the overlapping voxel detection and raw points exchange process, the adaptive voxelization process, and the optimization process, as introduced in \ref{sec.distributed_framework}.

\begin{figure} [ht]
	\centering
	\includegraphics[width=\linewidth]{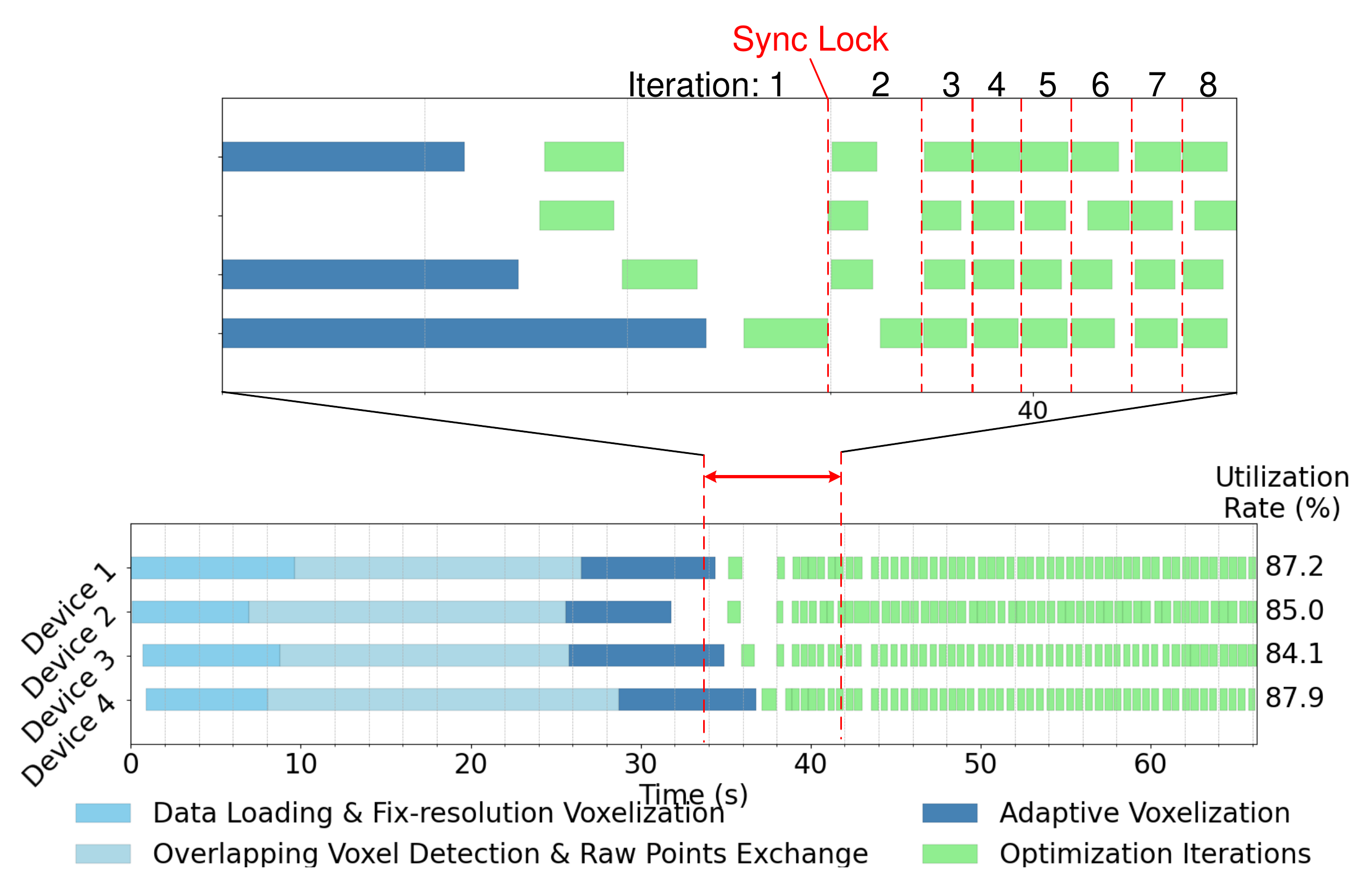}
	\caption{Time consumption of each device in the distributed bundle adjustment experiment. A zoomed-in view is provided at the top to display more details.}
	\label{fig.distributed_time_consumption}
\end{figure}

The results show that the entire bundle adjustment process takes approximately 66 seconds. They also reveal a phenomenon where each device experiences idle periods without any calculations, represented by the blank sections in each bar.
This phenomenon arises from the synchronization requirements during optimization and the unbalanced processing times across devices.
In the distributed bundle adjustment framework, a time synchronization lock is employed: each device can advance to the next optimization iteration only after all devices have updated and synchronized their poses, as illustrated in the detailed view in Figure \ref{fig.distributed_time_consumption}.
In the experiment, the devices were equipped with different types of CPUs, and the amount of data allocated to each device also varied. These factors led to different processing times for each device at each step. Consequently, devices with faster processing speeds were forced to wait for slower devices to complete their tasks. As a result, the overall time required for distributed bundle adjustment is determined by the device with the longest processing time.

During the evaluation, we calculated the utilization rate for each device, defined as the ratio of its effective processing time to its total processing time. The results, shown on the right side of Figure \ref{fig.distributed_time_consumption}, indicate that although Device 3 is equipped with the most powerful CPU (Intel i9-14900HX), it has the lowest utilization rate (84.1\%).
This observation highlights the importance of balancing the computational load across devices. The amount of data allocated to each device should be proportional to its computational capability.

\subsubsection{Bandwidth Evaluation}
We then evaluated and reported the network bandwidth load for each device, as shown in Figure \ref{fig.bandwidth_usage}.
The bandwidth usage per second of each device across time was recorded and presented. The results indicate that the primary bandwidth usage occurs in two distinct phases.
The first phase involves overlapping voxel detection and the exchange of raw point cloud data, as described in Section \ref{sec.distributed_framework}. Since overlapping voxel detection requires only the transmission of compressed data, the primary bandwidth usage during this phase stems from the exchange of raw point cloud data. Although maximum bandwidth may be required during this stage to transmit the data, it only lasts for approximately 17 seconds and occurs only once.
The second phase corresponds to the pose update synchronization process (also illustrated in Section \ref{sec.distributed_framework}). This phase involves the fast iterative synchronization of optimized poses. Since the scan poses contain only a small amount of data, each synchronization iteration (depicted as a peak during the pose update sync period in Figure \ref{fig.bandwidth_usage}) is brief and consumes minimal bandwidth. As a result, the overall bandwidth usage remains within a reasonable range. The total amount of data transmitted by each device is significantly smaller than the size of the point cloud it stores—specifically, 6.9\%, 2.9\%, 5.5\%, and 4.0\% for devices 1 through 4, respectively.

\begin{figure} [ht]
	\centering
	\includegraphics[width=\linewidth]{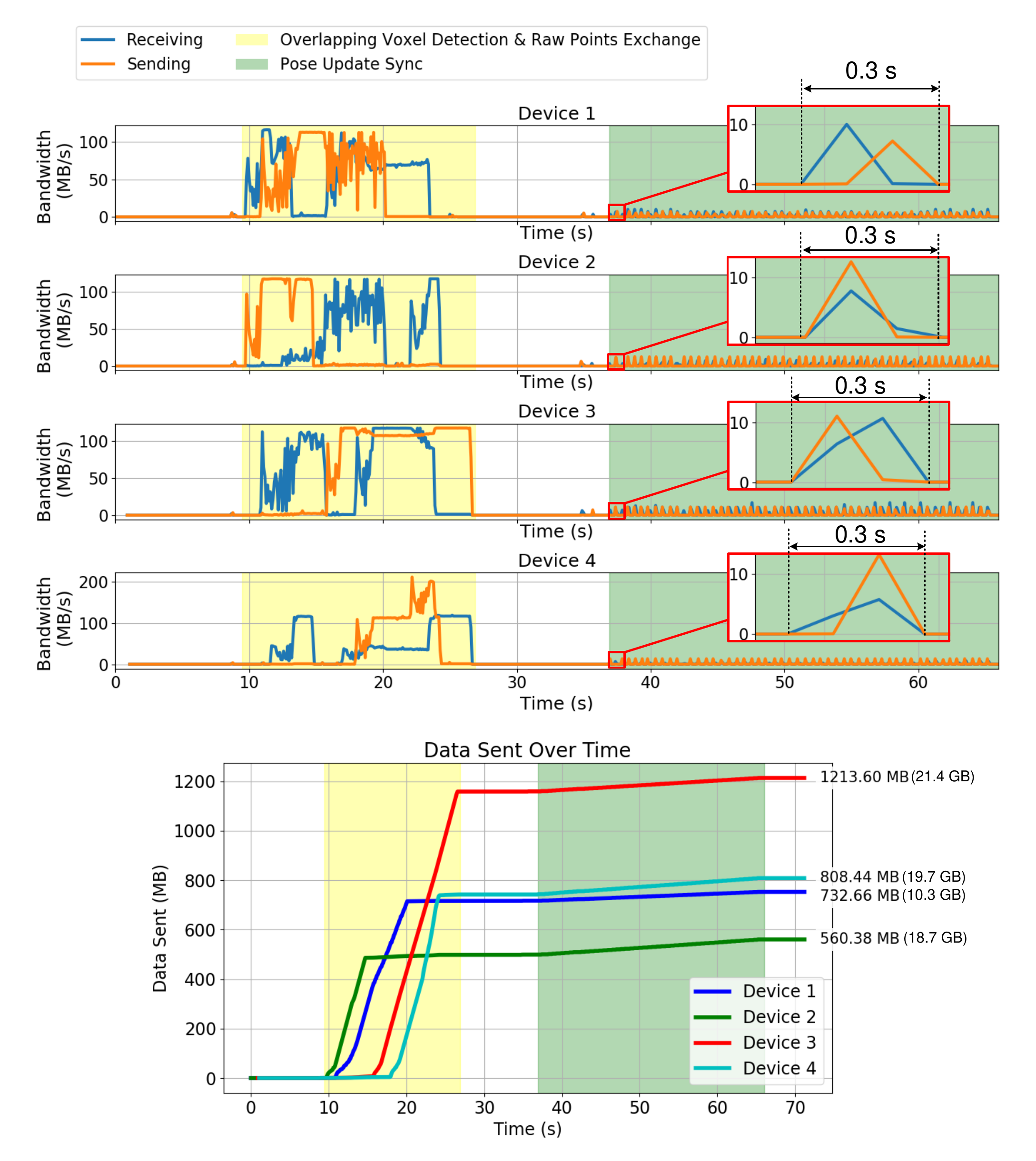}
	\caption{Evaluation result of the bandwidth usage during distributed bundle adjustment experiment. The upper four plots illustrate the bandwidth usage for each device over time, while the lower plot depicts the total data size transmitted by each device. {On the right side of the lower plot, the final data size transmitted by each device is indicated, along with the total data size stored on each device (shown in parentheses).}}
	\label{fig.bandwidth_usage}
\end{figure}

\section{Conclusion and Future Works} \label{sec.conclusion}
This paper proposes a new method for efficient and distributed large-scale point cloud bundle adjustment. Utilizing a majorization-minimization algorithm, the proposed method minimizes the point-to-plane residual by optimizing a newly introduced upper surrogate cost function, in which the LiDAR pose is completely decoupled.

A key advantage of decoupling the LiDAR pose in the surrogate cost function is that it results in a block-diagonal Hessian matrix. This reduces the time complexity of solving linear equations in the Levenberg-Marquardt optimizer from cubic to linear, significantly decreasing the optimization time for large-scale data. Additionally, the decoupled LiDAR poses can be distributed and optimized across multiple devices, enabling tasks with substantial memory demands to be completed using several consumer-level devices with limited memory.

Comprehensive experiments, including simulations and tests on multiple real-world public datasets, were conducted. The results demonstrate that the proposed method converges to the same global optimum as its baseline method, achieving comparable accuracy while being up to 704 times faster in optimization speed. These findings validate the effectiveness of the proposed method. Moreover, the distributed bundle adjustment experiment demonstrated the ability of the proposed method to optimize extremely large-scale data using multiple computational devices.

The experimental results also suggest potential directions for future research. For instance, the accuracy evaluation in the benchmark experiments indicates that the current adaptive voxelization framework, designed to extract and match plane features, may be inadequate for large-scale mapping. In large-scale datasets, plane features can exhibit significant variations in scale, making it hard to determine the adaptive voxelization parameters in practice. Additionally, since the evaluation of the proposed method has emerged as the primary time bottleneck and is parallelizable, accelerating the process using GPUs presents a promising direction for future work. Moreover, the results from the distributed bundle adjustment framework emphasize the importance of achieving load balance across different devices.

Furthermore, the optimization framework and algorithm are extendable. For instance, the proposed method can be easily combined with GNSS/RTK measurements, which is crucial for large-scale mapping to mitigate accumulation errors. Additionally, the decentralized nature of the proposed distributed bundle adjustment framework makes it applicable to other multi-agent tasks.

{
\appendix
\section{Proof of Lemma 1} \label{proof.lemma1}
Let \( S = \sum_{i=1}^{n} x_i \). Then the inequality becomes:
\begin{align*}
    S^2 \ge -4 z^2 + 4 S z.
\end{align*}
Bringing all terms to the left-hand side, we get:
\begin{align*}
    S^2 - 4 S z + 4 z^2 \ge 0.
\end{align*}
Observe that the left-hand side is a perfect square:
\begin{align*}
    (S - 2 z)^2 \ge 0.
\end{align*}
Since the square of any real number is non-negative, the inequality holds for all \( z \in \mathbb{R} \). Equality holds when \( (S - 2 z)^2 = 0 \), that is, when \( S - 2 z = 0 \). Therefore,
\begin{align*}
    z^\star = \frac{1}{2} S = \frac{1}{2} \sum_{i=1}^{n} x_i.
\end{align*}
\section{Proof of Theorem 1} \label{proof.theorem1}
Based on Lemma 1, each cost item $c_i(\mathbf T)$ in the original cost function (1) has 
\scriptsize
\begin{align}
    c_i(\mathbf T)=&\lambda_{\text{min}}\left(\frac{1}{N_i} \mathbf P_i - \frac{1}{N_i^2} \mathbf v_i \mathbf v_i^T \right) \; \nonumber \\
    =& \min_{\| \mathbf u \| = 1} \mathbf u^T \left( \frac{1}{N_i} \mathbf P_i - \frac{1}{N_i^2} \mathbf v_i \mathbf v_i^T \right) \mathbf u \; \nonumber \\
    \leq & {\mathbf u_{\text{min}}^{(k)}}^T \left( \frac{1}{N_i} \mathbf P_i - \frac{1}{N_i^2} \mathbf v_i \mathbf v_i^T \right) \mathbf u_{\text{min}}^{(k)} \ \tag{\text{Setting $\mathbf u = \mathbf u_{\text{min}}^{(k)}$}}\; \nonumber \\
    =& \frac{1}{N_i} \sum_{j=1}^{M_p} {\mathbf u_{\text{min}}^{(k)}}^T {\mathbf P}_{ij} \mathbf u_{\text{min}}^{(k)} - \frac{1}{N_i^2} {\mathbf u_{\text{min}}^{(k)}}^T \bigg( \sum_{j=1}^{M_p} {\mathbf v}_{ij} \bigg) \bigg( \sum_{j=1}^{M_p} {\mathbf v}_{ij} \bigg)^T \mathbf u_{\text{min}}^{(k)} \; \nonumber \\
    =& \frac{1}{N_i} \sum_{j=1}^{M_p} {\mathbf u_{\text{min}}^{(k)}}^T {\mathbf P}_{ij} \mathbf u_{\text{min}}^{(k)} - \frac{1}{N_i^2} \left( \sum_{j=1}^{M_p} {\mathbf u_{\text{min}}^{(k)}}^T {\mathbf v}_{ij} \right)^2 \nonumber \\
    \leq & \frac{1}{N_i} \sum_{j=1}^{M_p} {\mathbf u_{\text{min}}^{(k)}}^T {\mathbf P}_{ij} \mathbf u_{\text{min}}^{(k)} - \frac{4}{N_i^2} \sum_{j=1}^{M_p} \Big( {\mathbf u_{\text{min}}^{(k)}}^T {\mathbf v}_{ij} \Big) z^{(k)} \; \nonumber \\
    & \qquad + \frac{4}{N_i^2} {z^{(k)}}^2 \tag{\text{Lemma 1}}\; \nonumber \\
    =& \sum_{j=1}^{M_p} \left( \frac{1}{N_i} {\mathbf u_{\text{min}}^{(k)}}^T {\mathbf P}_{ij} \mathbf u_{\text{min}}^{(k)} - \frac{4 z^{(k)}}{N_i^2} {\mathbf u_{\text{min}}^{(k)}}^T {\mathbf v}_{ij}  \right) + \frac{4}{N_i^2} {z^{(k)}}^2 \; \nonumber \\
    \triangleq& c_{M_i}(\mathbf T | \mathbf T^{(k)}) 
\end{align}
\normalsize
Particularly, when $\mathbf T = \mathbf T^{(k)}$, we have
\footnotesize
\begin{align}
    c_{i} (\mathbf T^{(k)}) =& \lambda_{\text{min}}\left(\frac{1}{N_i} \mathbf P_i^{(k)} - \frac{1}{N_i^2} \mathbf v_i^{(k)} {\mathbf v_i^{(k)}}^T \right) \; \nonumber \\
    =&  {\mathbf u_{\text{min}}^{(k)}}^T \left( \frac{1}{N_i} \mathbf P_i^{(k)} - \frac{1}{N_i^2} \mathbf v_i^{(k)} {\mathbf v_i^{(k)}}^T \right) \mathbf u_{\text{min}}^{(k)} \nonumber \\
    =& \frac{1}{N_i} \sum_{j=1}^{M_p} {\mathbf u_{\text{min}}^{(k)}}^T {\mathbf P}_{ij}^{(k)} \mathbf u_{\text{min}}^{(k)} - \frac{1}{N_i^2} \left( \sum_{j=1}^{M_p} {\mathbf u_{\text{min}}^{(k)}}^T {\mathbf v}_{ij}^{(k)} \right)^2
\end{align}
\normalsize
where $\mathbf P_i^{(k)}, \mathbf v_i^{(k)}, \mathbf P_{ij}^{(k)}, \mathbf v_{ij}^{(k)}$ are $\mathbf P_i, \mathbf v_i, \mathbf P_{ij}, \mathbf v_{ij}$, respectively,  being evaluated at $\mathbf T^{(k)}$, $\mathbf u_{\text{min}}^{(k)}$ is the eigenvector corresponding to the minimal eigenvalue of $\frac{1}{N_i} \mathbf P_i^{(k)} - \frac{1}{N_i^2} \mathbf v_i^{(k)} {\mathbf v_i^{(k)}}^T$.

On the other hand, 
\footnotesize
\begin{align}
    c_{M_i}(\mathbf T^{(k)}|\mathbf T^{(k)}) =& \Bigg( \sum_{j=1}^{M_p} \Big( \frac{1}{N_i} {\mathbf u_{\text{min}}^{(k)}}^T {\mathbf P}_{ij} \mathbf u_{\text{min}}^{(k)} - \frac{4 z^{(k)}}{N_i^2} {\mathbf u_{\text{min}}^{(k)}}^T {\mathbf v}_{ij}  \Big) \; \nonumber \\
    & \qquad + \frac{4}{N_i^2} {z^{(k)}}^2 \Bigg) \Bigg|_{\mathbf T = \mathbf T^{(k)}} \nonumber \\
    =&\sum_{j=1}^{M_p} \left( \frac{1}{N_i} {\mathbf u_{\text{min}}^{(k)}}^T {\mathbf P}_{ij}^{(k)} \mathbf u_{\text{min}}^{(k)} - \frac{4 z^{(k)}}{N_i^2} {\mathbf u_{\text{min}}^{(k)}}^T {\mathbf v}_{ij}^{(k)}  \right) \; \nonumber \\
    & \qquad + \frac{4}{N_i^2} {z^{(k)}}^2 \nonumber \\
    =& \frac{1}{N_i}  \sum_{j=1}^{M_p} {\mathbf u_{\text{min}}^{(k)}}^T {\mathbf P}_{ij}^{(k)} \mathbf u_{\text{min}}^{(k)} - \frac{4 z^{(k)}}{N_i^2} \sum_{j=1}^{M_p} {\mathbf u_{\text{min}}^{(k)}}^T {\mathbf v}_{ij}^{(k)} \; \nonumber \\
    & \qquad + \frac{4}{N_i^2} {z^{(k)}}^2. 
\end{align}
\normalsize
Substituting $z^{(k)} = \frac{1}{2} \sum_{j=1}^{M_p} {\mathbf u_{\text{min}}^{(k)}}^T {\mathbf v}_{ij}^{(k)}$, we have
\footnotesize
\begin{align}
    c_{M_i}(\mathbf T^{(k)}|\mathbf T^{(k)}) =& \frac{1}{N_i} \sum_{j=1}^{M_p} {\mathbf u_{\text{min}}^{(k)}}^T {\mathbf P}_{ij}^{(k)} \mathbf u_{\text{min}}^{(k)} - \frac{1}{N_i^2} \left( \sum_{j=1}^{M_p} {\mathbf u_{\text{min}}^{(k)}}^T {\mathbf v}_{ij}^{(k)} \right)^2, 
\end{align}
\normalsize
which is equal to $c_i(\mathbf T^{(k)})$. Therefore, we have
\footnotesize
\begin{align}
     c_{M_i}(\mathbf T|\mathbf T^{(k)})\geq c_i(\mathbf T), \qquad c_{M_i}(\mathbf T^{(k)}|\mathbf T^{(k)}) = c_i(\mathbf T^{(k)})
\end{align}
\normalsize
for $\forall \mathbf T \in SE(3) \times \cdots \times SE(3)$.

\section{{Proof of Theorem 2}} \label{proof.theorem2}
As our design, (\ref{eq:surrogate}) is the summation of a set of sub-functions where each is only related to one scan pose. In this case, there is:
\scriptsize
\begin{align}
    \frac{\partial}{\partial \mathbf T} c_{M_i}(\mathbf T\mid\mathbf T^{(k)}) =& \frac{\partial}{\partial \mathbf T} \Bigg( \sum_{j=1}^{M_p} \left( \frac{1}{N_i} {\mathbf u_{\min}^{(k)}}^T {\mathbf P}_{ij} \mathbf u_{\min}^{(k)} - \frac{4 z^{(k)}}{N_i^2} {\mathbf u_{\min}^{(k)}}^T {\mathbf v}_{ij}  \right) \; \nonumber \\
    & + \frac{4}{N_i^2} {z^{(k)}}^2 \Bigg) \; \nonumber \\
    =& \sum_{j=1}^{M_p} \frac{\partial}{\partial \mathbf T} \underbrace{ \left(\frac{1}{N_i} {\mathbf u_{\min}^{(k)}}^T {\mathbf P}_{ij} \mathbf u_{\min}^{(k)} - \frac{4 z^{(k)}}{N_i^2} {\mathbf u_{\min}^{(k)}}^T {\mathbf v}_{ij}  \right)}_{f(\mathbf T_j \mid \mathbf T^{(k)})} \; \nonumber \\
    =& \begin{bmatrix}
        \cdots& \frac{\partial}{\partial \mathbf T_j} f(\mathbf T_j \mid \mathbf T^{(k)})& \cdots
    \end{bmatrix} \; \nonumber \\
\end{align}
\normalsize
Since $\frac{\partial}{\partial \mathbf T_j} f(\mathbf T_j \mid \mathbf T^{(k)})$ is still only {related} to {$\mathbf T_j$}, which implies that
\footnotesize
\begin{align}
    \frac{\partial^2}{\partial \mathbf T_k \partial \mathbf T_j} f(\mathbf T_j \mid \mathbf T^{(k)}) = \begin{cases}
    \frac{\partial^2}{(\partial \mathbf T_j)^2} f(\mathbf T_j \mid \mathbf T^{(k)}) & k = j, \\
    0 & k \neq j.
    \end{cases}
\end{align}
\normalsize
Thus, the second-order derivative of $c_{M_i}(\mathbf T\mid\mathbf T^{(k)})$ is diagonal:
\footnotesize
\begin{align}
    \frac{\partial^2}{\partial \mathbf T^2} c_{M_i}(\mathbf T\mid\mathbf T^{(k)}) =& \text{diag}\left(\cdots, \frac{\partial^2}{\partial \mathbf T_j^2} f(\mathbf T_j \mid \mathbf T^{(k)}), \cdots \right)
\end{align}
\normalsize
We then analyze the first-order and second-order derivatives of the $f(\mathbf T_j \mid \mathbf T^{(k)})$ related to $\mathbf T_j$, we first expand and reform $f(\mathbf T_j \mid \mathbf T^{(k)})$ as:
\scriptsize
\begin{align}
    f(\mathbf T_j\mid\mathbf T^{(k)})=&\frac{1}{N_i} {\mathbf u_{\min}^{(k)}}^T {\mathbf P}_{ij} \mathbf u_{\min}^{(k)} - \frac{4 z^{(k)}}{N_i^2} {\mathbf u_{\min}^{(k)}}^T {\mathbf v}_{ij} \; \nonumber \\
    =&\frac{1}{N_i} {\mathbf u_{\text{min}}^{(k)}}^T \left(\mathbf R_j \mathbf P_{f_{ij}} \mathbf R_j^T + \mathbf t_j \mathbf v_{f_{ij}}^T \mathbf R_j^T + \mathbf R_j \mathbf v_{f_{ij}} \mathbf t_j^T + N_{ij} \mathbf t_j \mathbf t_j^T \right) \mathbf u_{\text{min}}^{(k)} \; \nonumber \\
    & - \frac{4z^{(k)}}{N_i^2}{\mathbf u_{\text{min}}^{(k)}}^T \left( \mathbf R_j \mathbf v_{f_{ij}} + N_{ij} \mathbf t_j \right) \; \nonumber \\
    =& \underbrace{ \frac{1}{N_i} {\mathbf u_{\text{min}}^{(k)}}^T \mathbf R_j \mathbf P_{f_{ij}} \mathbf R_j^T \mathbf u_{\text{min}}^{(k)} - \frac{4z^{(k)}}{N_i^2} {\mathbf u_{\text{min}}^{(k)}}^T \mathbf R_j \mathbf v_{f_{ij}} }_{g(\mathbf R_j)} \; \nonumber \\
    & + \underbrace{\frac{2}{N_i} {\mathbf u_{\text{min}}^{(k)}}^T \mathbf R_j \mathbf v_{f_{ij}} \mathbf t_j^T \mathbf u_{\text{min}}^{(k)}}_{h(\mathbf R_j, \mathbf t_j)} \; \nonumber \\
    &+ \underbrace{\frac{N_{ij}}{N_i} {\mathbf u_{\text{min}}^{(k)}}^T \mathbf t_j \mathbf t_j^T \mathbf u_{\text{min}}^{(k)} - \frac{4z^{(k)}N_{ij}}{N_i^2} {\mathbf u_{\text{min}}^{(k)}}^T \mathbf t_j}_{l(\mathbf t_j)}
\end{align}
\normalsize
We then analyze the first-order and second-order derivatives of $g(\mathbf{R}_j)$, $h(\mathbf{R}_j, \mathbf{t}_j)$, and $l(\mathbf{t}_j)$ by applying the SE(3) perturbation described in (\ref{eq:input_perturbation}).
We begin with $g(\mathbf{R}_j)$, which depends solely on $\mathbf{R}_j$.
Considering the approximation $\exp\left( \lfloor \delta \boldsymbol{\phi}_j \rfloor \right) \approx \mathbf{I} + \lfloor \delta \boldsymbol{\phi}_j \rfloor + \frac{1}{2} \lfloor \delta \boldsymbol{\phi}_j \rfloor^2$, it can be shown that
\scriptsize
\begin{align}
    g(\mathbf R_j \boxplus & \delta \boldsymbol \phi_j) \; \nonumber \\
    =& \frac{1}{N_i} {\mathbf u_{\text{min}}^{(k)}}^T {\exp\left( \lfloor \delta \boldsymbol{\phi}_j \rfloor \right) \mathbf R_j} \mathbf P_{f_{ij}} {\mathbf R_j}^T \exp\left( \lfloor \delta \boldsymbol{\phi}_j \rfloor \right)^T \mathbf u_{\text{min}}^{(k)} \; \nonumber \\
    & - \frac{4z^{(k)}}{N_i^2} {\mathbf u_{\text{min}}^{(k)}}^T \exp\left( \lfloor \delta \boldsymbol{\phi}_j \rfloor \right) {\mathbf R_j} \mathbf v_{f_{ij}} \; \nonumber \\
    \approx& \frac{1}{N_i} {\mathbf u_{\text{min}}^{(k)}}^T \left( \mathbf I + \lfloor \delta \boldsymbol{\phi}_j \rfloor + \frac{1}{2} \lfloor \delta \boldsymbol{\phi}_j \rfloor^2 \right) {\mathbf R_j} \mathbf P_{f_{ij}} {\mathbf R_j}^T \; \nonumber \\ 
    & \left( \mathbf I - \lfloor \delta \boldsymbol{\phi}_j \rfloor + \frac{1}{2} \lfloor \delta \boldsymbol{\phi}_j \rfloor^2 \right) \mathbf u_{\text{min}}^{(k)} \; \nonumber \\
    &- \frac{4z^{(k)}}{N_i^2} {\mathbf u_{\text{min}}^{(k)}}^T \left( \mathbf I + \lfloor \delta \boldsymbol{\phi}_j \rfloor + \frac{1}{2} \lfloor \delta \boldsymbol{\phi}_j \rfloor^2 \right) {\mathbf R_j} \mathbf v_{f_{ij}} \; \nonumber \\
    =& \underbrace{\frac{1}{N_i} {\mathbf u_{\text{min}}^{(k)}}^T {\mathbf R_j} \mathbf P_{f_{ij}} {\mathbf R_j}^T \mathbf u_{\text{min}}^{(k)} - \frac{4z^{(k)}}{N_i^2} {\mathbf u_{\text{min}}^{(k)}}^T {\mathbf R_j} \mathbf v_{f_{ij}}}_{g_0} \; \nonumber \\
    & + \underbrace{\bigg( \frac{2}{N_i} {\mathbf u_{\text{min}}^{(k)}}^T {\mathbf R_j} \mathbf P_{f_{ij}} {\mathbf R_j}^T \lfloor \mathbf u_{\text{min}}^{(k)} \rfloor}_{\mathbf g_\phi} \; \nonumber \\ 
    & \qquad \underbrace{ + \frac{4z^{(k)}}{N_i^2} {\mathbf u_{\text{min}}^{(k)}}^T  \lfloor {\mathbf R_j} \mathbf v_{f_{ij}} \rfloor \bigg)}_{\mathbf g_\phi} \delta \boldsymbol{\phi}_j \; \nonumber \\
    &+ \frac{1}{2} \delta \boldsymbol \phi_j^T \underbrace{\bigg(-\frac{2}{N_i} \lfloor \mathbf u_{\text{min}}^{(k)} \rfloor {\mathbf R_j} \mathbf P_{f_{ij}} {\mathbf R_j}^T \lfloor \mathbf u_{\text{min}}^{(k)} \rfloor }_{\mathbf g_{\phi \phi}} \; \nonumber \\
    & \qquad \underbrace{ + \frac{1}{N_i} \lfloor \mathbf u_{\text{min}}^{(k)} \rfloor \lfloor \mathbf R_j \mathbf P_{f_{ij}} {\mathbf R_j}^T \mathbf u_{\text{min}}^{(k)} \rfloor}_{\mathbf g_{\phi \phi}} \; \nonumber \\
    & \qquad \underbrace{+ \frac{1}{N_i} \lfloor \mathbf R_j \mathbf P_{f_{ij}} {\mathbf R_j}^T \mathbf u_{\text{min}}^{(k)} \rfloor \lfloor \mathbf u_{\text{min}}^{(k)} \rfloor}_{\mathbf g_{\phi \phi}} \; \nonumber \\
    & \underbrace{- \frac{2 z^{(k)}}{N_i^2} \lfloor \mathbf u_{\text{min}}^{(k)} \rfloor \lfloor \mathbf R_j \mathbf v_{f_{ij}} \rfloor - \frac{2 z^{(k)}}{N_i^2} \lfloor \mathbf R_j \mathbf v_{f_{ij}} \rfloor \lfloor \mathbf u_{\text{min}}^{(k)} \rfloor \bigg)}_{\mathbf g_{\phi \phi}} \delta \boldsymbol \phi_j \; \nonumber \\
    &+ o \left( \delta \boldsymbol \phi_j^3 \right)
\end{align}
\normalsize
We then discuss $l(\mathbf t_j)$ with the same SE(3) perturbation.
\scriptsize
\begin{align}
    l( \exp( \lfloor \delta& \boldsymbol{\phi}_j\rfloor) \mathbf t_j+\delta \mathbf t_j) \; \nonumber \\
    =& \frac{N_{ij}}{N_i} (\exp \left( \lfloor \delta \boldsymbol \phi_j \rfloor \right) \mathbf t_j + \delta \mathbf t_j)^T \mathbf u_{\text{min}}^{(k)} {\mathbf u_{\text{min}}^{(k)}}^T ( \exp \left( \lfloor \delta \boldsymbol \phi_j \rfloor \right) \mathbf t_j + \delta \mathbf t_j) \; \nonumber \\
    &- \frac{4z^{(k)}N_{ij}}{N_i^2} {\mathbf u_{\text{min}}^{(k)}}^T ( \exp \left( \lfloor \delta \boldsymbol \phi_j \rfloor \right) \mathbf t_j + \delta \mathbf t_j) \; \nonumber \\
    \approx & \frac{N_{ij}}{N_i} \left( \left( \mathbf I + \lfloor \delta \boldsymbol{\phi}_j \rfloor + \frac{1}{2} \lfloor \delta \boldsymbol{\phi}_j \rfloor^2 \right) \mathbf t_j + \delta \mathbf t_j \right)^T \mathbf u_{\text{min}}^{(k)} \; \nonumber \\ 
    & \qquad \qquad{\mathbf u_{\text{min}}^{(k)}}^T \left( \left( \mathbf I + \lfloor \delta \boldsymbol{\phi}_j \rfloor + \frac{1}{2} \lfloor \delta \boldsymbol{\phi}_j \rfloor^2 \right) \mathbf t_j + \delta \mathbf t_j \right) \; \nonumber \\
    &- \frac{4z^{(k)}N_{ij}}{N_i^2} {\mathbf u_{\text{min}}^{(k)}}^T \left( \left( \mathbf I + \lfloor \delta \boldsymbol{\phi}_j \rfloor + \frac{1}{2} \lfloor \delta \boldsymbol{\phi}_j \rfloor^2 \right) \mathbf t_j + \delta \mathbf t_j \right) \; \nonumber \\
    =& \underbrace{\frac{N_{ij}}{N_i} {\mathbf t_j}^T {\mathbf u_{\text{min}}^{(k)}} {\mathbf u_{\text{min}}^{(k)}}^T {\mathbf t_j}  - \frac{4z^{(k)}N_{ij}}{N_i^2} {\mathbf u_{\text{min}}^{(k)}}^T {\mathbf t_j}}_{l_0} \; \nonumber \\
    &+ \underbrace{\left(  \frac{2 N_{ij}}{N_i} {\mathbf t_j}^T {\mathbf u_{\text{min}}^{(k)}} {\mathbf u_{\text{min}}^{(k)}}^T  - \frac{4z^{(k)}N_{ij}}{N_i^2} {\mathbf u_{\text{min}}^{(k)}}^T \right)}_{\mathbf l_t} \delta \mathbf t_j \; \nonumber \\
    &+ \frac{1}{2} \delta \mathbf t_j^T \underbrace{ \left( \frac{2 N_{ij}}{N_i} {\mathbf u_{\text{min}}^{(k)}} {\mathbf u_{\text{min}}^{(k)}}^T \right) }_{\mathbf l_{tt}} \delta \mathbf t_j \; \nonumber \\ 
    &+ \underbrace{\left( \frac{2 N_{ij}}{N_i}  {\mathbf u_{\text{min}}^{(k)}}^T \mathbf t_j {\mathbf t_j}^T \lfloor {\mathbf u_{\text{min}}^{(k)}} \rfloor - \frac{4z^{(k)} N_{ij}}{N_i^2} { \mathbf t_j }^T \lfloor \mathbf u_{\text{min}}^{(k)} \rfloor \right)}_{\mathbf l_{\boldsymbol \phi}} \delta \boldsymbol \phi_j \; \nonumber \\
    &+ \frac{1}{2} \delta \boldsymbol \phi_j^T \underbrace{ \bigg( -\frac{2N_{ij}}{N_i} \lfloor \mathbf t_j \rfloor \mathbf u_{\text{min}}^{(k)} {\mathbf u_{\text{min}}^{(k)}}^T \lfloor \mathbf t_j \rfloor + \frac{N_{ij}}{N_i} \lfloor \mathbf t_j \rfloor \lfloor \mathbf u_{\text{min}}^{(k)} {\mathbf u_{\text{min}}^{(k)}}^T \mathbf t_j \rfloor }_{\mathbf l_{\boldsymbol \phi \boldsymbol \phi}} \; \nonumber \\
    & \quad \underbrace{+ \frac{N_{ij}}{N_i} \lfloor \mathbf u_{\text{min}}^{(k)} {\mathbf u_{\text{min}}^{(k)}}^T \mathbf t_j \rfloor \lfloor \mathbf t_j \rfloor- \frac{4 z^{(k)} N_{ij}}{N_i^2} \lfloor \mathbf u_{\text{min}}^{(k)} \rfloor \lfloor \mathbf t_j \rfloor \bigg) }_{\mathbf l_{\boldsymbol \phi \boldsymbol \phi}} \delta \boldsymbol \phi_j \; \nonumber \\
    & + \delta \boldsymbol \phi_j^T \underbrace{ \bigg( \frac{2 N_{ij}}{N_i} \lfloor \mathbf t_j \rfloor \mathbf u_{\text{min}}^{(k)} {\mathbf u_{\text{min}}^{(k)}}^T \bigg) }_{\mathbf l_{\phi t}} \delta \mathbf t_j + o(\delta \boldsymbol \phi_j^3)
\end{align}
\normalsize
Finally, we discuss the $h(\mathbf R_j, \mathbf t_j)$.
\scriptsize
\begin{align}
    h(\mathbf R_j \boxplus & \delta \boldsymbol \phi_j, \exp (\lfloor \delta \boldsymbol \phi_j \rfloor) \mathbf t_j + \delta \mathbf t_j) \; \nonumber \\
    =& \frac{2}{N_i} {\mathbf u_{\text{min}}^{(k)}}^T \exp \left( \lfloor \delta \boldsymbol \phi_j \rfloor \right) \mathbf R_j \mathbf v_{f_{ij}} \left( \exp \left( \lfloor \delta \boldsymbol \phi_j \rfloor \right) \mathbf t_j + \delta \mathbf t_j \right)^T \mathbf u_{\text{min}}^{(k)} \; \nonumber \\
    \approx & \frac{2}{N_i} {\mathbf u_{\text{min}}^{(k)}}^T \left(\mathbf I + \lfloor \delta \boldsymbol \phi_j \rfloor + \frac{1}{2} \lfloor \delta \boldsymbol \phi_j \rfloor^2 \right) \mathbf R_j \mathbf v_{f_{ij}} \; \nonumber \\ 
    & \qquad \qquad \left( \left(\mathbf I + \lfloor \delta \boldsymbol \phi_j \rfloor + \frac{1}{2} \lfloor \delta \boldsymbol \phi_j \rfloor^2 \right) \mathbf t_j + \delta \mathbf t_j \right)^T \mathbf u_{\text{min}}^{(k)} \; \nonumber \\
    =& \underbrace{\frac{2}{N_i} {\mathbf u_{\text{min}}^{(k)}}^T \mathbf R_j \mathbf v_{f_{ij}} {\mathbf t_j}^T \mathbf u_{\text{min}}^{(k)}}_{h_0} + \underbrace{ \left( \frac{2}{N_i} {\mathbf u_{\text{min}}^{(k)}}^T \mathbf R_j \mathbf v_{f_{ij}} {\mathbf u_{\text{min}}^{(k)}}^T \right) }_{\mathbf h_t} \delta \mathbf t_j \; \nonumber \\
    & + \underbrace{\bigg( \frac{2}{N_i} {\mathbf u_{\text{min}}^{(k)}}^T \mathbf t_j \mathbf v_{f_{ij}}^T \mathbf R_j^T \lfloor \mathbf u_{\text{min}}^{(k)} \rfloor }_{\mathbf h_\phi} \; \nonumber \\
    & \qquad \underbrace{ + \frac{2}{N_i} {\mathbf u_{\text{min}}^{(k)}}^T \mathbf R_j \mathbf v_{f_{ij}} {\mathbf t_j}^T \lfloor \mathbf u_{\text{min}}^{(k)} \rfloor \bigg)}_{\mathbf h_\phi} \delta \boldsymbol \phi_j \; \nonumber \\
    &+ \frac{1}{2} \delta \boldsymbol \phi_j^T \underbrace{\bigg( \frac{2}{N_i} \lfloor {\mathbf u_{\text{min}}^{(k)}} \rfloor \lfloor \mathbf R_j \mathbf v_{f_{ij}} {\mathbf t_j}^T \mathbf u_{\text{min}}^{(k)} \rfloor }_{\mathbf h_{\phi \phi}} \; \nonumber \\
    & \qquad \underbrace{ + \frac{2}{N_i} \lfloor \mathbf t_j \mathbf v_{f_{ij}}^T \mathbf R_j^T \mathbf u_{\text{min}}^{(k)} \rfloor \lfloor {\mathbf u_{\text{min}}^{(k)}} \rfloor}_{\mathbf h_{\phi \phi}} \; \nonumber \\
    & \qquad \underbrace{- \frac{4}{N_i} \lfloor {\mathbf u_{\text{min}}^{(k)}} \rfloor \mathbf R_j \mathbf v_{f_{ij}} \mathbf t_j^T \lfloor \mathbf u_{\text{min}}^{(k)} \rfloor \bigg) }_{\mathbf h_{\phi \phi}} \delta \boldsymbol \phi_j \; \nonumber \\
    & + \delta \boldsymbol \phi_j^T \underbrace{ \left( - \frac{2}{N_i} \lfloor {\mathbf u_{\text{min}}^{(k)}} \rfloor \mathbf R_j \mathbf v_{f_{ij}} {\mathbf u_{\text{min}}^{(k)}}^T \right) }_{\mathbf h_{\phi t}} \delta \mathbf t_j + o \left( \delta \boldsymbol \phi_j^2 \delta \mathbf t_j \right)
\end{align}
\normalsize
We then combine $g(\mathbf{R}_j)$, $h(\mathbf{R}_j, \mathbf{t}_j)$, and $l(\mathbf{t}_j)$ to compute the final perturbation result of $f(\mathbf{T}_j \mid \mathbf{T}_j^{(k)})$.
\footnotesize
\begin{align}
    f(\mathbf t_j \boxplus \delta \mathbf T_j \mid \mathbf T^{(k)}) =& f_0 + \mathbf J_j \delta \mathbf T_j + \frac{1}{2} \delta \mathbf T_j^T \mathbf H_j \delta \mathbf T_j + o(\delta \mathbf T_j^3) \; \nonumber \\
\end{align}
\normalsize
where
\footnotesize
\begin{align}
    f_0 =& g_0 + h_0 + l_0 \; \nonumber \\
    \mathbf J_j =& \begin{bmatrix}
        \mathbf g_\phi + \mathbf h_\phi + \mathbf l_{\phi} & \mathbf h_t + \mathbf l_t
    \end{bmatrix} \; \nonumber \\
    \mathbf H_j =& \begin{bmatrix}
        \mathbf g_{\phi \phi} + \mathbf h_{\phi \phi} + \mathbf l_{\phi \phi} & \mathbf h_{\phi t} + \mathbf l_{\phi t} \\
        (\mathbf h_{\phi t} + \mathbf l_{\phi t})^T & \mathbf l_{tt}
    \end{bmatrix}
\end{align}
\normalsize
which also implies that
\footnotesize
\begin{align}
\label{eq.jacob_and_hess}
    \frac{\partial c_{M_i}(\mathbf T \mid \mathbf T^{(k)})}{\partial \mathbf T} &= \begin{bmatrix}
        \mathbf J_0 & \mathbf J_1 & \cdots &  \mathbf J_{M_p}
    \end{bmatrix} \; \nonumber \\
    \frac{\partial^2 c_{M_i}(\mathbf T \mid \mathbf T^{(k)})}{(\partial \mathbf T)^2} &= {\mathrm{diag}}(\mathbf H_1, \mathbf H_2, \cdots , \mathbf H_{M_p})
\end{align}
\normalsize
And the Hessian matrix is block diagonal.
\section{{Proof of Corollary 1}} \label{proof.corollary1}
For the $i$-th dimension $\xi$ in state $\mathbf{T}$, the first-order partial derivatives of $c_M(\mathbf T |\mathbf T^{(k)})$ and $c(\mathbf T)$ with respect to $\xi$ at $\mathbf T^{(k)}$ are defined as the following limits:
\scriptsize
\begin{align} \label{eq:limits}
    \left( \frac{\partial c_M(\mathbf T |\mathbf T^{(k)})}{\partial \xi} \right) \left( \mathbf T^{(k)} \right) &= \lim_{\delta \xi \rightarrow 0}\frac{c_M(\mathbf T^{(k)} \boxplus \delta \mathbf T |\mathbf T^{(k)}) - c_M(\mathbf T^{(k)} |\mathbf T^{(k)})}{\delta \xi} \; \nonumber \\
    \left( \frac{\partial c(\mathbf T)}{\partial \xi} \right) \left( \mathbf T^{(k)} \right) &= \lim_{\delta \xi \rightarrow 0}\frac{c(\mathbf T^{(k)} \boxplus \delta \mathbf T) - c(\mathbf T^{(k)})}{\delta \xi}
\end{align}
\normalsize
where $\delta \mathbf T = [0, \cdots, \delta \xi, \cdots, 0]$ and is nonzero only at the $i$-th dimension.

Then we consider the one-sided derivatives of (\ref{eq:limits}). When we approach $\mathbf{T}^{(k)}$ from the positive direction, i.e., $\delta \xi \rightarrow 0^+$, we obtain:
\scriptsize
\begin{align} \label{eq.right_lim}
     \left( \frac{\partial c_M(\mathbf T |\mathbf T^{(k)})}{\partial \xi} \right) \left( \mathbf T^{(k)} \right) =& \lim_{\delta \xi \rightarrow 0^+}\frac{c_M(\mathbf T^{(k)} \boxplus \delta \mathbf T |\mathbf T^{(k)}) - c_M(\mathbf T^{(k)} |\mathbf T^{(k)})}{\delta \xi} \; \nonumber \\
    =& \lim_{\delta \xi \rightarrow 0^+}\frac{c_M(\mathbf T^{(k)} \boxplus \delta \mathbf T |\mathbf T^{(k)}) - c(\mathbf T^{(k)})}{\delta \xi} \; \nonumber \\
    \geq& \lim_{\delta \xi \rightarrow 0^+}\frac{c(\mathbf T^{(k)} \boxplus \delta \mathbf T) - c(\mathbf T^{(k)})}{\delta \xi} \; \nonumber \\
    =& \left( \frac{\partial c(\mathbf T)}{\partial \xi} \right) \left( \mathbf T^{(k)} \right)
\end{align}
\normalsize
When we approach $\mathbf{T}^{(k)}$ from the negative direction, i.e. $\delta \xi \rightarrow 0^-$, we obtain:
\scriptsize
\begin{align} \label{eq.left_lim}
     \left( \frac{\partial c_M(\mathbf T |\mathbf T^{(k)})}{\partial \xi} \right) \left( \mathbf T^{(k)} \right) =& \lim_{\delta \xi \rightarrow 0^-}\frac{c_M(\mathbf T^{(k)} \boxplus \delta \mathbf T |\mathbf T^{(k)}) - c_M(\mathbf T^{(k)} |\mathbf T^{(k)})}{\delta \xi} \; \nonumber \\
    =& \lim_{\delta \xi \rightarrow 0^-}\frac{c_M(\mathbf T^{(k)} \boxplus \delta \mathbf T | \mathbf T^{(k)}) - c(\mathbf T^{(k)})}{\delta \xi} \; \nonumber \\
    \leq& \lim_{\delta \xi \rightarrow 0^-}\frac{c(\mathbf T^{(k)} \boxplus \delta \mathbf T) - c(\mathbf T^{(k)})}{\delta \xi} \; \nonumber \\
    =& \left( \frac{\partial c(\mathbf T)}{\partial \xi} \right) \left( \mathbf T^{(k)} \right)
\end{align}
\normalsize
This implies that:
\scriptsize
\begin{align} \label{eq.partial_eq}
     \left( \frac{\partial c_M(\mathbf T |\mathbf T^{(k)})}{\partial \xi} \right) \left( \mathbf T^{(k)} \right) = \left( \frac{\partial c(\mathbf T)}{\partial \xi} \right) \left( \mathbf T^{(k)} \right)
\end{align}
\normalsize
Since (\ref{eq.partial_eq}) holds for any dimension of state $\mathbf T$, we finally have (\ref{eq.jacob_eq}).
}

\bibliographystyle{SageH}
\bibliography{reference}
% \includepdf[pages=1-]{supplementary.pdf}

\end{document}